%% file: SYLPH.tex
\documentclass[preprint,12pt]{elsarticle}

\usepackage{amssymb}
\usepackage[ruled,linesnumbered]{algorithm2e}
\usepackage{color}
\usepackage{amsmath}
\usepackage{mathtools}
\usepackage{soul}
\usepackage{booktabs}
\usepackage{amsthm}
\usepackage{float}
\usepackage{subfigure} 
\usepackage{hyperref}

\journal{Artificial Intelligence}

\begin{document}

\begin{frontmatter}


\cortext[cor1]{Corresponding author at: Department of Mechanical Engineering, College of Design and Engineering, National University of Singapore, 21 Lower Kent Ridge Rd, Singapore}

\title{Social Behavior as a Key to Learning-based Multi-Agent Pathfinding Dilemmas}

\author[1]{Chengyang He}
\author[1]{Tanishq Duhan}
\author[1]{Parth Tulsyan}
\author[1]{Patrick Kim}
\author[1]{Guillaume Sartoretti\corref{cor1}}
\ead{guillaume.sartoretti@nus.edu.sg}

\address[1]{{Department of Mechanical Engineering, College of Design and Engineering, National University of Singapore},
            {21 Lower Kent Ridge Rd}, 
            {117575}, 
            {Singapore}}


\begin{abstract}

The Multi-agent Path Finding (MAPF) problem involves finding collision-free paths for a team of agents in a known, static environment, with important applications in warehouse automation, logistics, or last-mile delivery. To meet the needs of these large-scale applications, current learning-based methods often deploy the same fully trained, decentralized network to all agents to improve scalability. However, such parameter sharing typically results in homogeneous behaviors among agents, which may prevent agents from breaking ties around symmetric conflict (e.g., bottlenecks) and might lead to live-/deadlocks. In this paper, we propose SYLPH, a novel learning-based MAPF framework aimed to mitigate the adverse effects of homogeneity by allowing agents to learn and dynamically select different social behaviors (akin to individual, dynamic roles), without affecting the scalability offered by parameter sharing. Specifically, SYLPH agents learn to select their Social Value Orientation (SVO) given the situation at hand, quantifying their own level of selfishness/altruism, as well as an SVO-conditioned MAPF policy dictating their movement actions. To these ends, each agent first determines the most influential other agent in the system by predicting future conflicts/interactions with other agents. Each agent selects its own SVO towards that agent, and trains its decentralized MAPF policy to enact this SVO until another agent becomes more influential. To further allow agents to consider each others' social preferences, each agent gets access to the SVO value of their neighbors. As a result of this hierarchical decision-making and exchange of social preferences, SYLPH endows agents with the ability to reason about the MAPF task through more latent spaces and nuanced contexts, leading to varied responses that can help break ties around symmetric conflicts. Our comparative experiments show that SYLPH achieves state-of-the-art performance, surpassing other learning-based MAPF planners in random, room-like, and maze-like maps, while our ablation studies demonstrate the advantages of each component in SYLPH. We finally experimentally validate our trained policies on hardware in three types of maps, showing how SYLPH allows agents to find high-quality paths under real-life conditions. Our code and videos are available at: \href{http://marmotlab.github.io/mapf_sylph}{marmotlab.github.io/mapf\_sylph}.

\end{abstract}




\end{frontmatter}



\section{Introduction}
\label{Intro}

Multi-Agent Path Finding (MAPF) involves devising collision-free paths for multiple agents within a known and static space, guiding them from their current positions to their respective goals~\cite{stern2019multi}.
MAPF is commonly used in warehouse automation~\cite{li2021lifelong, wang2020mobile}, air traffic control~\cite{polydorou2021learning}, autonomous driving~\cite{li2023intersection}, and video game AI~\cite{ma2017feasibility}.
While the objectives for coordinating the behaviors of all agents may differ among these scenarios, such as minimizing makespan, ensuring safety margins, or optimizing resource usage, a common feature is the need to deploy and coordinate always-increasing numbers of agents.

\begin{figure}[h]
\centering
\includegraphics[width=5in]{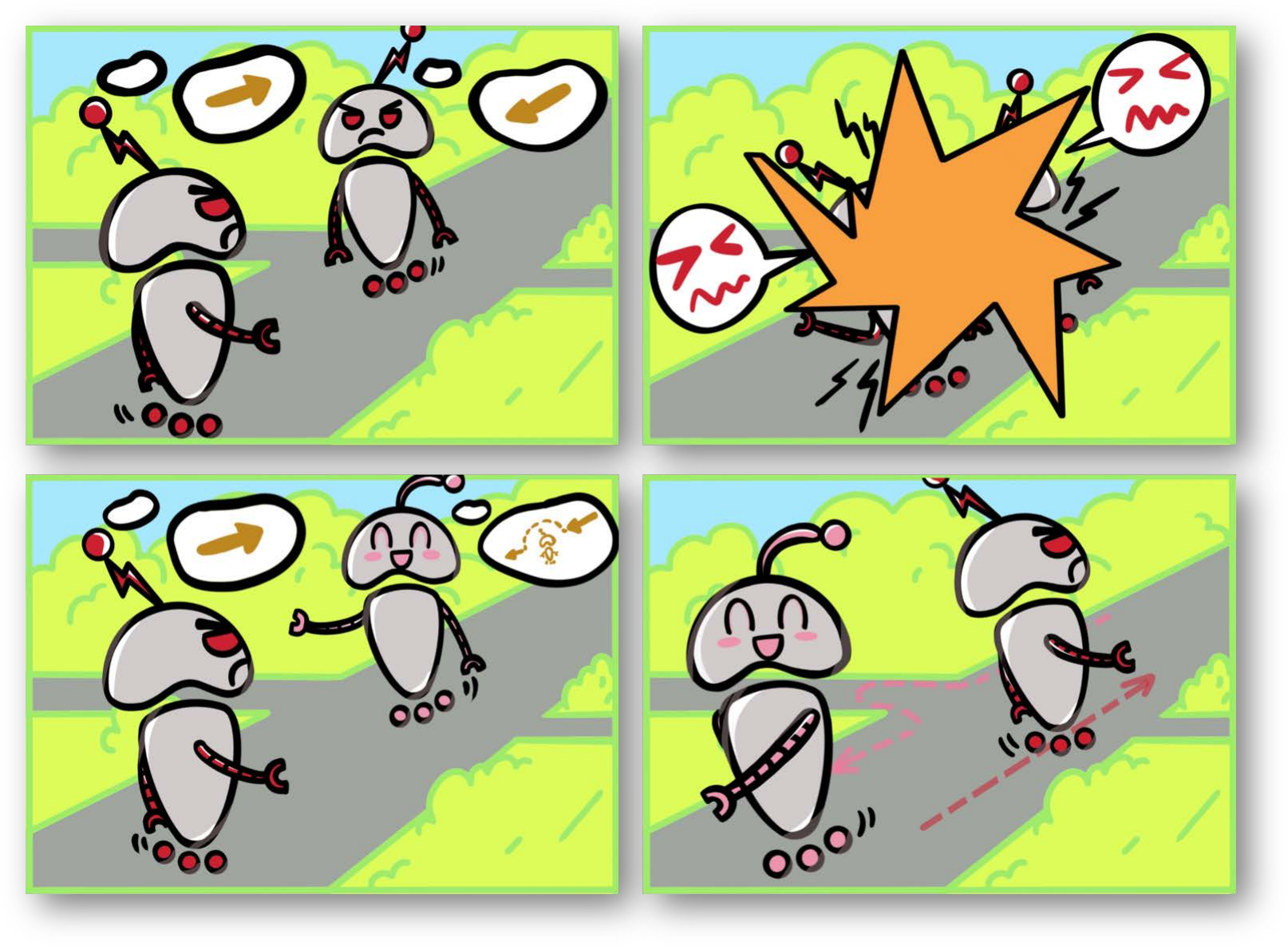}
\vspace{-0.3cm}
\caption{A simple example illustrates the difference between a completely selfish team and a team with diverse social roles. 
The two figures above show that when facing a symmetric challenge, a team of selfish agents falls into a social dilemma. 
In contrast, agents with different SVOs can more easily achieve cooperation by breaking the homogeneity of their behavior patterns. 
They benefit from a combination of individualism and pro-socialism within the team, as shown in the two figures below.}
\label{comic_summary}
\vspace{-0.4cm}
\end{figure}

To meet the needs of such large-scale applications, the community has increasingly turned to learning-based methods~\cite{sartoretti2019primal,lin2023sacha,yan2024neural}  that leverage deep reinforcement learning (DRL) and advanced neural network architectures. 
These decentralized, reactive MAPF planners offer improved scalability, addressing the limitations of traditional algorithms that struggle with the curse of dimensionality as team sizes increase~\cite{ferner2013odrm, ma2019searching}, though they often result in suboptimal solutions.
To enhance scalability, learning-based MAPF algorithms usually adopt the Independent Policy Learning (IPL) paradigm~\cite{tan1993multi}, where agents consider/observe each other as dynamic features of the environment, greatly reducing the complexity of learning and increasing the robustness of policies.
The learning-based MAPF methods benefit from good scalability also partly due to decentralized decision-making driven by \textit{parameter sharing}~\cite{damani2021primal,wang2020mobile,guan2022ab}, in which experiences collected by multiple agents are combined to train a single neural network during the training process, and the fully trained network is then deployed to all agents for execution.
However, independent learning with parameter sharing tends to homogeneize the behavior of the agents, resulting in all agents exhibiting similar social preferences, often individualistic due to their need to maximize individual rewards and complete individual tasks. 
In highly structured scenarios, such as bottlenecks and narrow corridors, homogeneous behaviors may cause agents to fall into symmetric social dilemmas~\cite{li2021pairwise}, which can lead to live- or deadlocks, as illustrated in Fig.~\ref{comic_summary}.
These \textit{social dilemmas} arise from symmetries in the environment / agents' states and are exacerbated by conflicts between individual interests, where none of the agents involved can obtain higher rewards unless compromises are made by one of them. 
Another limitation of this training approach is that the independent nature of learning results in agents that cannot easily encourage/exhibit coordinated maneuvers, which limits their performance in dense scenarios. 
A viable solution for agents to address these two issues is by learning social behaviors, specifically through reasoning about and balancing short-term self-interests with long-term team benefits.
Learning social behavior equips agents with varying levels of prosociality, breaking homogeneity and by more tightly coupling agents directly in reward space. 
This approach help with coordinated maneuvers, essential for resolving social dilemmas in large-scale MAPF instances.

\begin{figure}[t]
\centering
\includegraphics[width=5.4in]{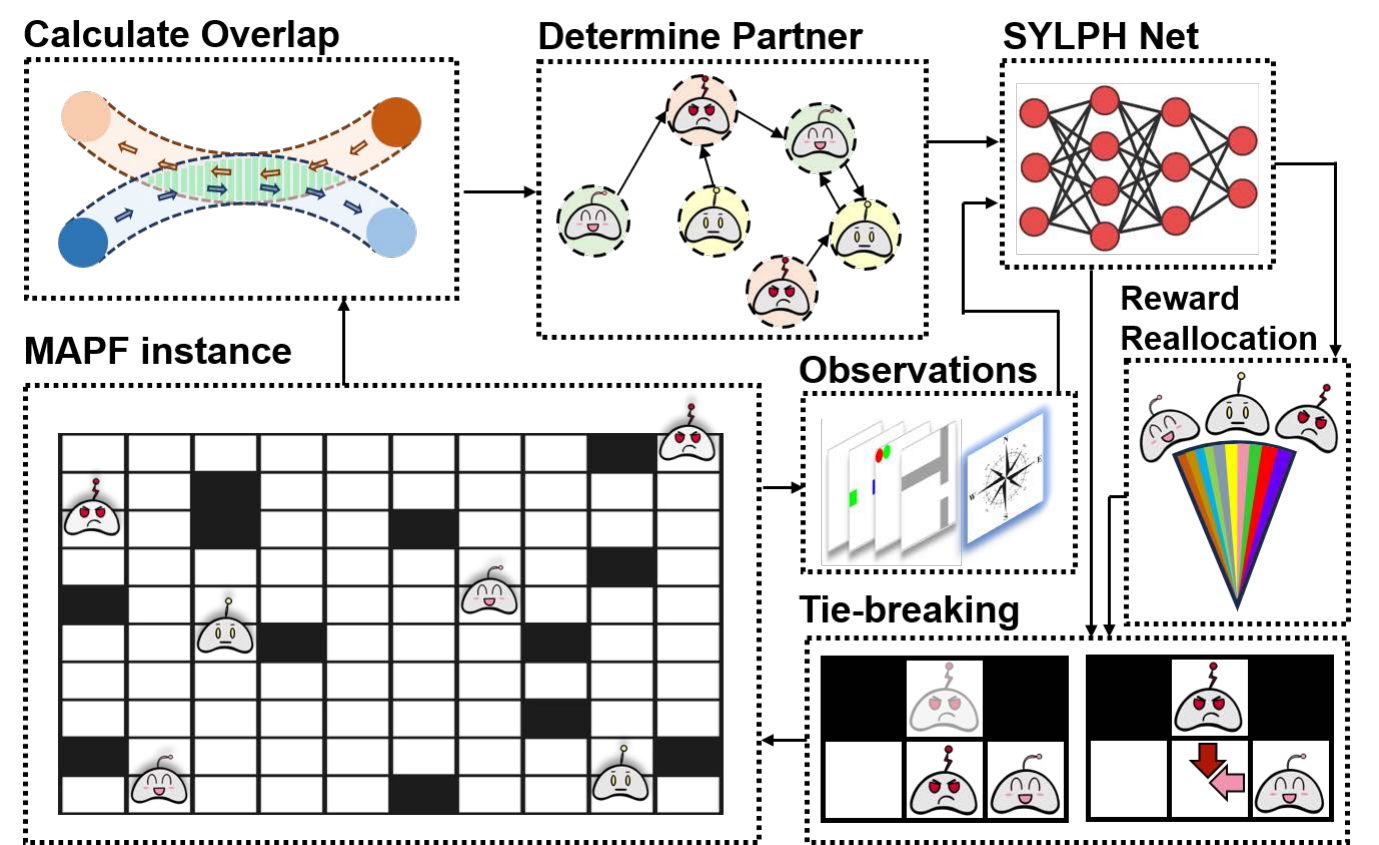}
\vspace{-0.3cm}
\caption{The key components and overall architecture of SYLPH. By introducing social preference into the MAPF framework as a temporary extension skill, the agent is equipped with social behavior to better cope with social dilemmas such as symmetry problems and blocking problems.}
\label{graphic_abstract}
\vspace{-0.4cm}
\end{figure}

To these ends, this paper develops a novel learning-based hierarchical MAPF framework, SYLPH (SociallY-aware multi-agent PatHfinding), which helps agents mitigate the adverse effects of homogeneity on scalability by allowing them to learn dynamic roles while following the parameter-sharing paradigm. 
Leveraging a tool from social psychology, we propose incorporating the notion of Social Value Orientation (SVO) as a latent social preference space on top of the agents' action space, to develop a hierarchical decision-making architecture (as shown in Fig.~\ref{graphic_abstract}).
The lower-level action space consists of control commands (four cardinal directions and staying idle), allowing agents to interact with the environment by choosing different discrete movements.
In contrast, the latent social preference space functions as an upper-level latent space, serving as conditions for low-level SVO-based conditional policies that indirectly influence these cardinal direction choices. 
In our framework, SVO acts as a social metric that characterizes an individual's prosocial level and weighs their reward against the rewards of others~\cite{schwarting2019social}.
To this end, each agent needs to predict the potential interaction/conflict degree with other agents based on the current configuration and determine the most influential one in the system (referred as \textit{partner}) about its own pathfinding task. 
The agent's SVO will act on this partner to promote social behavior formation between them until another agent becomes more influential. 
Additionally, to comprehensively consider the social preferences of other agents, SYLPH agents can access the SVOs of those with whom they have relationships in a multi-hop manner for communication learning.
By knowing other team members' SVOs, agents can infer each other's potential behavior and make decisions that enhance team performance.
Due to the larger latent space provided our hierarchical decision framework, along with the contextual information from agents' exchange of social preferences and historical SVO memory, we show that SYLPH agents can overcome the decision-making challenges posed by symmetric environments more easily.
Moreover, the coordinated reward mechanisms introduced by SVO incentivize agents to make decisions that benefit the entire team rather than just maximizing their own rewards. 
This reduces the individualistic tendencies fostered by parameter sharing and independent training, fostering coordinated maneuvers.

Furthermore, the upper-level latent spaces in this paper have clear meanings, explicitly representing the social preferences of agents and providing interpretable insights into their decision-making processes.
A benefit of explicit SVO representation is that it improves the interpretability and predictability of agent behaviors. 
By characterizing agents based on their SVO (i.e., whether they are selfish, neutral, or pro-social), we can explicitly model their preferences and priorities in decision-making processes.
This differentiation allows us to better understand and anticipate the strategies and interactions of various agents, leading to more transparent and explainable outcomes.
To validate the advantages mentioned above, we conduct a comprehensive evaluation of the proposed framework.
We test MAPF configurations with various team sizes on three different types of maps: random maps, room-like maps, and maze maps.
Additionally, we also present an ablation study to verify the effectiveness of each component of our framework.
Finally, through real-robot experiments on three different types of maps, we demonstrate that the paths provided by SYLPH are executable under real-life conditions.


\section{Prior Work}


\subsection{Traditional MAPF Methods}

The research community's exploration of MAPF originated with traditional methods.
Traditional MAPF techniques are categorized by optimality into three types: optimal, bounded suboptimal, and unbounded suboptimal~\cite{gao2023review}.
Optimal methods, by definition, enable multiple agents to achieve their goals with minimal overall cost.
Theoretically, optimal methods are typically complete; they should offer a solution as long as a solution exists for the MAPF instance.
The two most influential optimal planners are M*~\cite{wagner2011m} and CBS~\cite{sharon2015conflict}, and many methods in the community are extensions and improvements of these two methods~\cite{ferner2013odrm,chan2021ecbs,ma2019searching}. 
When M* does not detect conflicts between agents, the state space only expands by one cell at each timestep, following the optimal choices of all agents. 
For conflicting agents, M* evaluates combinations of their possible actions while attempting to balance these with the optimal actions of other agents. 
It does so by searching through the joint space of agents around collisions, and at worst can fall back onto exhaustive search for the whole team.
CBS adopts a two-layer approach, where the upper layer uses conflict-based binary tree search and the lower layer uses the optimal single-agent planner A* to provide each agent with an optimal path based on conflict constraints.

Bounded suboptimal MAPF planners are designed to handle the increased computational load encountered by optimal solvers on expansive maps and with large-scale teams.
These planners strike a balance between solution quality and computational efficiency by introducing a suboptimality factor, which can be tightened to make the planner approximate the optimal solution.
For instance, inflated M*~\cite{pearl1984heuristics,wagner2015subdimensional} achieves a relaxation of M* by altering the heuristic function of the A* algorithm.
ECBS implements both levels of CBS as a focal search~\cite{barer2014suboptimal}, reducing the number of collisions and accelerating the CBS search process.
Building on ECBS, an advanced planner called EECBS~\cite{li2021eecbs} has been developed, which combines explicit estimation search at the high level with focal search at the low level.
Thanks to its integration of multiple improvement techniques, EECBS maintains good performance even with large-scale teams.

To further pursue solver effectiveness rather than optimality, the community gradually began to study unbounded suboptimal solvers. 
This type of method's focus shifts from achieving the optimal solution to enhancing success rates and solution speeds without strict adherence to optimality, which is particularly useful in highly complex scenarios.
The current state-of-the-art MAPF solvers also belong to this method\footnote{We mark the state-of-the-art algorithms with \textcolor{red}{$\star$}.}. 
MAPF-LNS~\cite{li2021anytime} designs a new framework by combining small-scale high-quality solvers and large-scale low-quality solvers. 
First, any efficient MAPF algorithm is used to find the initial solution for the instance. 
From there, a Large Neighborhood Search (LNS)~\cite{shaw1998using} is used to re-plan the subgroup of agents to improve the quality of the solution.
In MAPF-LNS2\textcolor{red}{$^\star$}~\cite{li2022mapf}, an opposite idea is adopted. 
Many collision-allowed paths are generated at first, and subgroups of conflicting paths are then continuously selected and updated within a limited time until they are collision-free.
PIBT is a traditional one-step update method based on priority~\cite{okumura2022priority}.
It generates a single step of the path for each agent at every timestep until the problem is solved or a preset maximum number of steps is reached. 
LaCAM\textcolor{red}{$^\star$}~\cite{okumura2023lacam}, building upon PIBT, is a more advanced two-layer MAPF planner. 
At the higher level, it searches a sequence of configurations, where a configuration is a tuple of locations for all agents. 
The generation of these configurations is a low-level task. 
PIBT, as a low-level planner, can generate configurations that satisfy constraints extremely quickly.

The progression from optimal to unbounded suboptimal solvers in the MAPF community reflects a shift towards more flexible and scalable solutions that are capable of managing the increasing complexity of practical applications.
This trend underscores a growing emphasis on algorithms that can deliver reasonable solutions within acceptable time frames, especially under the constraints of large-scale environments and agent populations.


\subsection{Learning-based MAPF Methods}

Deep learning has been a promising tool to solve one-shot MAPF problems ever since PRIMAL was introduced~\cite{sartoretti2019primal}.
Recent works have shown that agents can achieve better cooperation through communication, as it allows access to richer information.
Therefore, some algorithms tried using either Graph Neural Networks~\cite{li2020graph,li2021message,ma2021learning} or Transformers~\cite{wang2023scrimp} to incite communication learning.
Alternatively, a communication approach predicts agents' priorities using traditional algorithms and integrates this priority into an ad-hoc routing protocol for prioritized communication learning~\cite{li2022multi}.
These communication learning methodologies leverage distinct information aggregation mechanisms to facilitate information exchange among team agents, enriching the decision-making process and enhancing team cooperation. 
However, the exchanged messages lack explicit meaning since they are products of learned encoding.
This renders the communication process both uninterpretable and unpredictable~\cite{chen2022nagphormer,kim2020communication}. 
Such opacity can hinder the model's ability to manage the quality and effectively aggregate the communicated messages, producing an excess of redundant messages as a result. 
These messages will fail to significantly increase the performance, and only impose additional computational loads.
While some efforts have been made to filter communication recipients~\cite{ding2020learning,ma2021learning}, the inherent black-box nature of communication learning process precludes a clear understanding of the possible phenomena.

Another possible way to significantly enhance the performance of MAPF algorithms is to enable agents to access global information, a statement that has been supported by many existing algorithms.
Considering that most MAPF application scenarios occur in known environments, using only the information accessed by the agents' local field of view (FoV) is a waste of resources. 
However, to maximize the advantage of the full environment information, it is crucial to encode and represent global information properly. 
The academic community devised various ways to accomplish this.
Some approaches use the A* algorithm to compute individual paths for agents and integrate these into the agents' observation~\cite{liu2020mapper} and reward structures~\cite{wang2020mobile}.
The approach described in~\cite{liu2020mapper} mandates that agents strictly adhere to the A* path and imposes penalties for any deviation, even if it is minor. 
Conversely,~\cite{wang2020mobile} significantly softens these constraints, permitting agents to stray from the A* path with the provision that they can rejoin at any future point to receive the cumulative rewards they previously earned.
Alternative to end-goal expert path guidance, a more flexible and deployable heuristic map representations emerged with the use of Breadth-First Search (BFS)~\cite{ma2021distributed}. 
Rather than strictly enforcing adherence to A* paths, these heuristic maps offer agents a range of actions on all unoccupied cells, guiding them towards their goal in a flexible manner. 
This provision of global information proves advantageous in densely populated environments, as it affords the agent a broader spectrum of choices.
A global map encoding based on a graph transformer has been used in our recent work~\cite{he2023alpha}, preserving global information to the greatest extent and being more expressive. 
This method furnishes the agent with insights at the global graph level without compelling adherence to any intermediate nodes, thus affording the agent increased flexibility and capabilities for long-horizon planning.

Learning-based planners are sometimes prone to deadlocks or livelocks due to unforeseen circumstances. 
Implementing post-processing techniques to assist agents in overcoming such cases can also boost algorithm performance. 
Upon detecting a conflict within the agent's learned policy,~\cite{virmani2021subdimensional} utilizes M*~\cite{wagner2015subdimensional} for re-planning within the joint configuration space.
However, this approach is highly resource-intensive, as M* calculates complete paths for all agents, with frequent calls leading to considerable computational burden.
Our previous research~\cite{wang2023scrimp} introduced a tie-breaking strategy grounded in state-value, offering a more efficient alternative by enabling local conflict resolution. 
This post-processing operations are carried out upon conflict detection to prevent collisions. 
Nonetheless, this technology can result in a dramatic surge in computational demands in densely populated scenarios.
While such interventions can enhance the overall performance of the algorithm, they reveal a limitation at the model level: the trained model itself lacks long-term foresight to avoid conflict.
To guarantee scalability, the aforementioned current methods deploy a unified network across all agents. 
These approaches, however, engender homogeneity and self-interest among team members. 
This precipitates social dilemmas in specific scenarios such as symmetric cases. 

To address this challenge, this work introduces a new SVO-based social behavior learning mechanism, titled SYLPH. 
By eliminating the need for configuring separate networks for different agents and introducing diversity within the team, SYLPH effectively mitigate the adverse effects of homogeneity caused by parameter sharing on scalability.
Furthermore, our tie-breaking mechanism is based on agents' SVO, enabling our model to enhance the conflict-aware foresight of the generated policies without necessitating extra computational resources for post-processing.


\subsection{Social Preference Usage}

In social psychology, Social Value Orientation (SVO) is a common metric for encapsulating an individual's social preferences and propensities~\cite{mcclintock1989social}, which serves as an indicator of a person's inclination to cooperate with fellow team members.
Specifically, SVO quantifies the extent to which an individual prioritizes personal versus team benefits.
It can be represented by the angle $\phi$~\cite{murphy2011measuring}, as shown in Fig.~\ref{graphic_abstract}.
The individual is said to be more egoistic if the $\phi$ value is closer to 0, and more altruistic if the $\phi$ value is closer to 90.
In robotics, Schwarting et al. pioneered the application of this concept in the field of autonomous driving~\cite{schwarting2019social}.
They employed this metric as an intermediary variable, enhancing the accuracy of predictive models for human-driven vehicle trajectories.
Subsequent research in autonomous driving has expanded upon this concept.
Examples include investigations into social communication among multiple vehicles in the presence of adversaries~\cite{zhang2023zero}, enabling agents to autonomously adapt their SVO preferences~\cite{dai2023socially}.
Furthermore, the field of video games has also seen significant research on SVO~\cite{mckee2020social,madhushani2023heterogeneous}. 
These studies show that in scenarios involving social dilemmas, the diversity of SVOs within the population is beneficial.

Previous SVO-based approaches typically set pre-allocated, immutable SVOs to agents.
This arrangement is reasonable for tasks with a clear division of labor (e.g., Cleanup and Harvest~\cite{jaques2019social}) or shared objectives (e.g., StarCraft II and Google Research Football~\cite{li2021celebrating}), where the necessity for diverse roles to collaboratively contribute to the team's objective is evident. 
In such scenarios, roles are often predetermined based on prior knowledge and maintained throughout the task. 
The absence of any role type can cause the entire system to malfunction, making the problem unsolvable.
In MAPF tasks, however, an agent may need to adopt various behavioral patterns depending on the situation in order to achieve more effective solutions.
In other words, the allocation of roles is not individual-oriented but situation-oriented, and thus should remain dynamic throughout the task.
Furthermore, while there is a common overarching goal, each agent also pursues individual objectives in MAPF. 
In order to balance self-interest and social benefit, there is a need for agents to learn a flexible SVO that can adapt to environmental changes. 
Drawing inspiration from skill learning~\cite{eysenbach2018diversity,sharma2019dynamics,he2020skill}, we propose viewing an agent's SVO as a temporally extended skill that depends on the previous timestep's SVO and the SVOs of other agents.
The agent can then dynamically select an SVO based on observations in varying contexts, thereby influencing its lower-level action space decision-making.

In this paper, SYLPH adopts a flexible approach where agents can adaptively learn real-time SVO policies in response to their current environment, thereby moving away from the rigid and predefined allocation of social roles.


\section{Problem Statement}


\subsection{MAPF Problem Formulation}

The Multi-Agent Path Finding (MAPF) problem exhibits numerous variants, including classic one-shot MAPF~\cite{li2022mapf}, MAPF with kinematic constraints~\cite{ma2019lifelong}, lifelong MAPF~\cite{skrynnik2023learn,damani2021primal}, prioritized MAPF~\cite{chandra2023socialmapf}, multi-agent pickup and delivery~\cite{okumura2022priority}, and etc. 
This paper focuses on the classic one-shot MAPF problem. 
Characteristically, the classic MAPF instance is set up on an undirected simple graph $\mathcal{G}=(\mathcal{V},\mathcal{E})$, encompassing a set of agents $\mathcal{A}=\{a_1,a_2\cdots a_n\}$, a pre-settled start positions $\mathcal{S}=\{s_1, s_2, \cdots, s_n\}\in\mathcal{V}$, and a designated set of destinations/goals $\mathcal{D}=\{d_1,d_2\cdots d_n\}\in\mathcal{V}$, where $n$ denotes the number of agents.
In this context, time $t\in\mathbb{N}$ is treated as discrete, allowing an agent $a_i^t$ to either move to an adjacent vertex $a_i^{t+1}=v_j\in\mathcal{N}_\mathcal{G}(v_i)$ or remain stationary at its current vertex $a_i^{t+1}=v_i$ within a single timestep. 
The aim of the MAPF task is to generate collision-free paths for all agents from their respective initial occupied vertices $s_i$ to their goal vertices $d_i$ within the minimum possible number of timesteps. 
We consider a set of paths $\{\tau_i\}$ for agents $i=1,2,\cdots,n$, where each path $\tau_i$ is a sequence of vertices that agent $a_i$ traverses over time $t=0,1,\cdots,T$, with $T$ being the maximum timestep considered.
All the paths must satisfy both \textbf{Condition 1:} $\forall i,j\in \{1,2,\cdots,n\},i\neq j,\forall t\in\{0,1,\cdots,T\},\tau_i(t)\neq\tau_j(t)$ (no two agents occupy the same vertex at any time), \textbf{Condition 2:} $\forall i,j\in \{1,2,\cdots,n\},i\neq j,\forall t\in\{0,1,\cdots,T-1\},(\tau_i(t)\neq \tau_j(t+1))\vee (\tau_i(t+1)\neq \tau_j(t))$ (no two agents swap vertices between consecutive timesteps), and \textbf{Condition 3:} $\forall i\in \{1,2,\cdots,n\},((\tau_i(0)=s_i)\wedge((\tau_i(T)=d_i))$ (all agents start from the pre-settled position and reach their goal at the end of the task), to make sure the paths are collision-free and effective.


\subsection{Environment Type}

This paper evaluates the proposed framework's effectiveness across three distinct map types. 
Firstly, we consider random maps, a common choice among learning-based planners for both training and testing due to the large variance in its structure and complexity.
This variability challenges the model's generalization capabilities, as the unpredictability in obstacle placement necessitates the learning of a diverse array of policies. 
Secondly, we assess performance on room-like maps, which are more structured than random maps and feature elements such as doorways, narrow corridors, and rooms. 
These structured obstacles compel agents to develop long-horizon planning capabilities to effectively find the path. 
Furthermore, room-like maps often contain cut vertices that can amplify minor errors into significant team-wide setbacks, highlighting the importance of planner stability. 
Lastly, given that this paper is dedicated to solving social dilemmas in MAPF problems, we also conduct tests on maze maps. 
This map type is characterized by numerous long corridors, dead-ends, and various edge cases. 
Such features pose significant challenges for existing learning-based MAPF planners due to issues such as corridor symmetry and target symmetry~\cite{li2021pairwise}.
To effectively address this type of problem, agents require a high level of coordination.
SYLPH addresses these challenges and achieves better performance by introducing SVO for the agent, effectively breaking the symmetry.


\section{Social Behavior Learning}
\label{social_behavior_learning}

In this section, we delve into the integration of the Social Value Orientation (SVO) concept within the Multi-Agent Path Finding (MAPF) problem, introducing a novel MAPF framework named SYLPH that incorporates social preferences.
By enabling agents to learn and adopt SVO-based social preferences, we introduce diversity into the multi-agent system.
This equips agents with the capability to navigate and resolve social dilemmas, such as the various symmetries frequently encountered in MAPF challenges.
For example, coordination of agents with opposite goals in narrow passages and livelock caused by symmetric goal positions in open areas.
The diversity ensures that agents will exhibit distinct behaviors based on their individual SVOs even when placed in identical environments, leading to varied decision-making outcomes.

We explore this integration from two primary perspectives (as shown in Fig.~\ref{graphic_abstract}): the generation of hierarchical policies and the influence of upper level SVO policies on the formulation of action-oriented policies. 
This exploration aims to highlight how the incorporation of SVO not only enriches the policy depth available to agents but also enhances their problem-solving efficacy within the MAPF context, enabling a more nuanced and cooperative pathfinding.


\subsection{Partner Selection}
\label{partner_selection}

In this work, introducing SVO into the MAPF process is aimed at mitigating potential social dilemmas. 
A crucial initial question is identifying the origins of these dilemmas from the perspective of an agent, specifically determining with whom to collaborate to effectively solve them.
Since SVO is used to balance an agent's self-interest with collective interests, representing the collective interests becomes our first task.
An intuitive approach is to average the rewards of other members in the team or other members within a certain observation range as the collective interests.
However, there are two significant disadvantages to doing so. 
First, this mixed collective interest representation of multiple agents is biased from the perspective of the current agent. 
Because it may mix information from agents that are irrelevant to the social dilemma encountered by the current agent, which will confuse the agent. 
Second, it difficult for the neural network to establish relationships between the agent and other fellow agents under such a collective interest representation. 
Compressing too much information into an average reward causes the neural network to lose significant information, making it challenging to reason about the original relationships between agents.
Therefore, in this paper we let the agent choose a \textit{partner} for SVO determination.
The term \textit{partner} refers to this other specific agent involved in the dyadic team with the primary (ego) agent.
Additionally, it is also worth mentioning that the agent-partner pair is not bi-directional, which means an agent's partner may choose the third agent as its partner.
This approach draws parallel insights from some recent autonomous driving research~\cite{dai2023socially, yang2018mean}, which suggests that an autonomous vehicle's behavior is predominantly influenced by the vehicle in its immediate vicinity rather than others within the broader FoV. 
These studies have led us to realize that focusing on a single partner can sometimes be more advantageous than considering multiple agents, because it allows neural networks to more easily reason about the relationships between different team members.

\begin{algorithm}[t]
	\caption{Selection of Temporary Partner.}
	\label{algo_1}
	\KwIn{The grid obstacle map: $\mathcal{G}$; the set of all agents' current positions: $\mathcal{A}$; the set of all agents' goals: $\mathcal{D}$.}
	\KwOut{All agent's partner: $\mathcal{P}\in\mathbb{N}_{+}^{n\times 1}$; The potential overlap of optimal paths among all agents: $\mathcal{O}_L\in\mathbb{R}^{n\times n}$.}  
	\BlankLine
        Initialize $\mathcal{P}$ as $\emptyset$, and $\mathcal{O}_L$ as $\mathbf{0}^{n\times n}$;
        \hfill \textcolor{blue}{$\triangleright$ \texttt{single optimal paths}}
        
	Compute A* paths $\mathcal{T}_{A*}\leftarrow\{\mathcal{G},\mathcal{A},\mathcal{D}\}$ for all agents;

    \For{$\forall v_t^i \in \mathcal{T}_{A*}$}{
        Determine the direction of the agent: $\delta_{v}^i\leftarrow\{v_t^i, v_{t+1}^i\}$;
        \hfill \textcolor{blue}{$\triangleright$ \texttt{flow of the optimal paths}}
    }

    \For{$\forall v \in \mathcal{T}_{A*}$}{
        \If{$\exists v^i_{t_i}, v^j_{t_j}\equiv v~\textnormal{\textbf{and}}~i\neq j~\textnormal{\textbf{and}}~\delta_{v}^i\neq \delta_{v}^j$}{
            $\mathcal{O}_L[i,j] = \mathcal{O}_L[i,j] + \gamma_{ol}^{t_i} + \gamma_{ol}^{t_j}$;
            \hfill \textcolor{blue}{$\triangleright$ \texttt{overlap calculation}}\\
            $\mathcal{O}_L[j,i] = \mathcal{O}_L[j,i] + \gamma_{ol}^{t_i} + \gamma_{ol}^{t_j}$; 
        }
    }

    \ForEach{row $\in\mathcal{O}_L$ }{
        \eIf{$row = 0$}{
            $\mathcal{P}[\texttt{Index}(row)]=\texttt{Index}(row)$;
            \hfill \textcolor{blue}{$\triangleright$ \texttt{temporary partner}}
        }{
            $\mathcal{P}[\texttt{Index}(row)]=\mathop{\arg\max}\limits_{i}row[i]$;
        }
    }
    
\end{algorithm}

In order to find the most impactful partner among all agents, we formalize the method of selecting partners in Algorithm~\ref{algo_1}. 
Specifically, the procedure for selecting a temporary partner in the context of MAPF with an emphasis on SVO can be detailed in four steps:
\begin{itemize}
    \item Calculating Single Agent Optimal Paths [Line 1-2]: 
    Initially, for each agent within the system, an optimal path is computed using the $A*$ algorithm based on the current environmental configuration $(\mathcal{G},\mathcal{A},\mathcal{D})$. 
    Each agent's path is delineated as a sequence of vertex coordinates:
      \begin{equation}
      \begin{aligned}
      &\mathcal{T}_{A*} = \{\tau_{a*}^1, \cdots, \tau_{a*}^n\}\\
      &\tau_{a*}^i = \{v_0^i,v_1^i,\cdots,v_T^i\}~~~~~i=1,2,\cdots,n.
      \end{aligned}
      \end{equation}
    \item Determining the Flow of These Individual Paths [Line 3-5]:
    Based on the optimal path computed, the flow of the path can be determined from the sequence of vertex coordinates. 
    In other words, the agent's orientation $\delta_v^i$ at any cell can be obtained by comparing the vertex coordinates between the current $v_t^i$ and the next timestep $v_{t+1}^i$. 
    \item Computing the Overlap of Optimal Path Flows Between Agents [Line 6-11]: With the flow of individual paths established, the next step involves assessing the overlap between these flows.
    A key idea is that if the flows of agents are in the same direction at a cell, i.e. $\delta_v^i=\delta_v^j$, then we consider no potential conflicts at this cell between agents and therefore consider this cell's overlap as $0$.
    However, if the agents are not moving in the same direction at a cell (as shown in Fig.~\ref{flow_overlap}), indicating potential crossing points or interactions, the overlaps should be assessed according to Line [8-9].
    The impact of such overlaps on an agent's decision-making weakens as the distance between the overlapping cell's position and the agent's current location increases. 
    The decay of overlap impact based on distance introduces the concept of a decay factor $\gamma_{ol}\in (0, 1]$, which is a predefined hyper-parameter.
    This factor adjusts the significance of distant overlaps on an agent's current choices, with a higher value indicating a more far-sighted agent that considers distant overlaps more significantly, while a lower value indicating an agent more focused on immediate or nearby overlaps.
    Mathematically, the overlap between two agents' paths is quantified as the weighted sum of all overlapping cells within their path flows: 
    \begin{equation}
    \begin{aligned}
        \mathcal{O}_L[i,j] = \mathcal{O}_L[i,j] = \sum_{v\in(\tau_{a*}^i\land\tau_{a*}^j)} \gamma_{ol}^{\texttt{Index}_{\tau_{a*}^i}(v)} + \gamma_{ol}^{\texttt{Index}_{\tau_{a*}^j}(v)}
    \end{aligned}
    \end{equation}
    The weight of each overlapping cell is adjusted by the decay factor, relative to its distance from the agent's current position. 
    This method provides a nuanced approach to evaluating potential path conflicts, allowing agents to prioritize their immediate navigation decisions while still accounting for future interactions. 
    \item Finding a Temporary Partner [Line 12-17]: 
    Based on the overlap analysis, agents are then paired or assigned a temporary partner.
    If an agent's optimal path flow does not exhibit any overlap with the flows of other agents, the agent defaults to selecting itself as its partner. 
    This scenario indicates that the agent can proceed without the need to adjust its path in response to potential conflicts with others, allowing for path finding towards its goal without external coordination.
    Conversely, if there is an overlap between an agent's path flow and that of one or more other agents, the agent will choose as its temporary partner the agent with which it has the largest weighted overlap.           
\end{itemize}

\begin{figure}[h]
\centering
\includegraphics[width=4.5in]{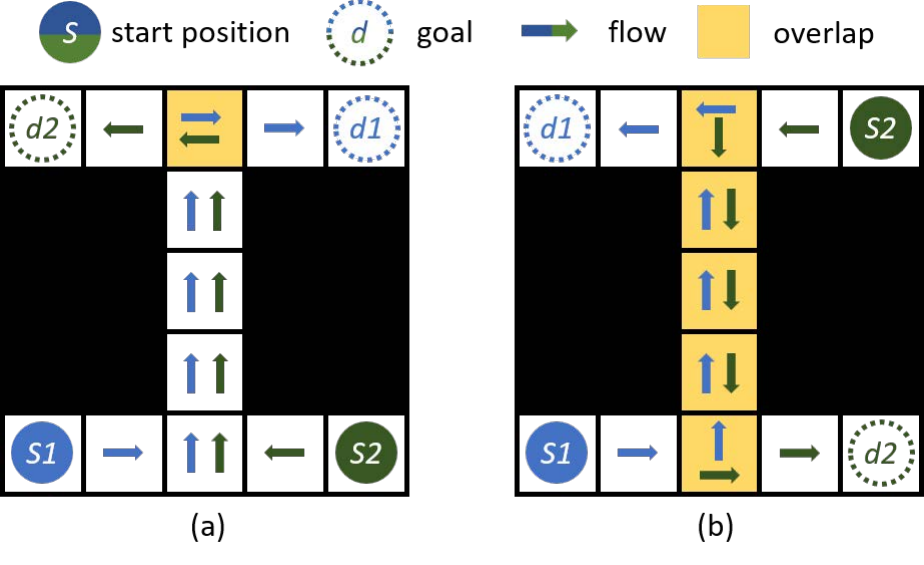}
\vspace{-0.3cm}
\caption{Overlap in optimal path flows between agents, caused by their varied starting and goal position configurations within the same map. 
In scenario (a), two agents traverse a narrow corridor moving in the same direction, which results in minimal conflict. 
Conversely, scenario (b) involves agents needing to navigate in opposite directions within the same space, significantly heightening the potential for conflict due to the direct opposition in their intended paths.
According to Algorithm~\ref{algo_1}, the calculated overlap in scenario (a) is markedly less than that in scenario (b). 
This distinction aligns well with intuitive expectations and the specific objectives of managing social dilemmas within multi-agent path finding. 
The larger overlap in scenario (b) suggests a higher degree of conflict and necessitates more critical intervention or strategy adjustment to avoid collision or deadlock, highlighting a situation of greater social distress.}
\label{flow_overlap}
\vspace{-0.4cm}
\end{figure}

After establishing the method above for selecting a temporary partner, the next step involves delineating the criteria for selecting and switching partners. 
We propose that updates to partner selection may not have to occur per timestep to ensure system stability and consistency in agent interactions.
The formal process for updating partners is outlined in Algorithm~\ref{algo_2}. 
Firstly, an agent selects a temporary partner at the start and treats this agent as its initial fixed partner. 
This fixed partnership remains unchanged until the overlap in the optimal path flow between the agent and its fixed partner is eliminated, indicating that any potential social dilemma or conflict with this particular partner has been resolved. 
By adopting this method, updates regarding the agents' partner occur asynchronously. 
Such a mechanism not only reflects a realistic and practical approach to managing interactions within a dynamic multi-agent environment but also significantly contributes to the overall stability of the system. 
This strategy allows for the focused resolution of conflicts with specific partners before considering a shift to new partner dynamics, ensuring that changes in partnerships are meaningful and based on resolved interactions rather than fluctuating frequently without resolving underlying conflicts.

\begin{algorithm}[t]
	\caption{Fixed Partner Update Criteria.}
	\label{algo_2}
	\KwIn{All agent's temporary partner: $\mathcal{P}\in\mathbb{N}_{+}^{n\times 1}$; The potential overlap of optimal paths among all agents: $\mathcal{O}_L\in\mathbb{R}^{n\times n}$.}
	\KwOut{All agents' fixed partners: $\mathcal{P}_{fix}\in\mathbb{N}_{+}^{n\times 1}$.}  
	\BlankLine

        Initialize $\mathcal{P}_{fix}$ according to the temporary partner $\mathcal{P}$ provided by Algorithm~\ref{algo_1} at the beginning of the episode;

        \For{$\forall\textnormal{agent}~i$}{
            \eIf{$\mathcal{O}_L[i,\mathcal{P}_{fix}[i]]=0$}{
                $\textnormal{\textbf{Update}}~\mathcal{P}_{fix}[i]~\textnormal{\textbf{as}}~\mathcal{P}[i]$;
                \hfill \textcolor{blue}{$\triangleright$ \texttt{update fixed partner}}
            }
            {
                $\textnormal{\textbf{Keep} the partner of}~i~\textnormal{unchanged}$.
                \hfill \textcolor{blue}{$\triangleright$ \texttt{keep fixed partner}}
            }
        }
        
\end{algorithm}

The algorithm introduced in this subsection quantifies the degree of conflict between agents by measuring the extent of their optimal path overlaps, thereby offering a formal method to express social dilemmas within a multi-agent environment. 
By precisely delineating the magnitude of these social dilemmas, the approach facilitates the identification of the agent that presents the greatest potential for conflict. 
This approach to clarifying social dilemmas enables the SVO mechanism to be applied more effectively. 
By focusing on resolving the most significant potential conflicts, agents can better navigate their environment, avoid collisions, and reach their goals in a more efficient manner. 


\subsection{SVO Generation}
\label{svo_policy_genration}

\subsubsection{Socially-aware reward}
Inspired by \textit{skill learning}~\cite{he2020skill,lowe2017multi}, our approach moves away from the way of assigning fixed, immutable SVO to agents.
Instead, we conceptualize SVO as a temporally extended skill - a dynamic attribute that does not reside with any single agent but rather emerges from the patterns of coordinated behavior among agents.
This reconceptualization recognizes that SVO is not a one-size-fits-all attribute.
This dynamic, learnable SVO enables agents to adjust their social preferences based on the current context and interactions with other agents.
Agents with dynamic SVO can change their behavior to meet the needs of the team, helping in avoiding live-/deadlocks and reduces the chances of prolonged conflicts, especially in highly structured maps.

In this light, the selection of an SVO is treated as an additional policy $\pi_\phi(z|z')$, akin to skill selection in skill learning frameworks.
The training of this SVO policy is conducted in tandem with the training of the agents' action policies, creating a synergistic relationship where both policies are interconnected through the mechanism of SVO. 
This linkage depends on partner selection as outlined in Section~\ref{partner_selection}, where the choice of partner directly influences the dynamics of the SVO policy.
Specifically, in our context, any agent and its partner must reach their respective goals due to task requirements. 
Since learning-based planners rely on parameter sharing to improve scalability, a self-interested reward structure is necessary to motivate the agent to pursue its own goal.
However, when an agent insists on following its optimal path, its corresponding partner may have to incur an additional external penalty to manage the arising conflict, embodying the zero-sum nature of their interaction. 
Therefore, we constraint that the agent's SVO, denoted by $Z$, should be between egoistic ($Z\approx 0^{\circ}$) and prosocial ($Z\rightarrow 45^{\circ}$). 
When $Z$ satisfies this restriction and all external rewards are non-positive (consistent with our previous research), $R_i^a$ is monotonically non-increasing within the domain of $[0^{\circ}, 45^{\circ}]$ \footnote{Proof can be found in~\ref{proof_1}}. 
It implies that in cases devoid of conflict, agents are predisposed towards egoistic behavior, thereby maximizing their own long-term cumulative rewards $R_i^a$. 
Conversely, in situations where conflicts exist, there is a tendency for agents to adopt a more prosocial demeanor to maximize $R_i^s$. 
Such a shift facilitates the achievement of superior long-term rewards for the group formed by the agent and its partner, highlighting the utility of SVO in mediating self-sacrifice for collective gain.
Based on the above discussion, the reward structures of the SVO policy and action policy are as follows:
\begin{equation}\label{reward_redistri}
\begin{aligned}
    R_i^s &\coloneqq (R_i^{ex} + R_p^{ex})~/~\rho,~~i,p \in\mathcal{A} \\
    R_i^a &\coloneqq \cos{Z}\cdot R_i^{ex} + \sin{Z} \cdot R_p^{ex}, ~~Z\thicksim \pi_\phi(z|z').
\end{aligned}
\end{equation}
where, $R_i^{ex}$ denotes the external reward for agent $i$, adhering to the reward configuration established in our prior research~\cite{sartoretti2019primal,wang2023scrimp,he2023alpha}. 
The structure of the external reward is encapsulated in the outcomes of the agent's interactions with the environment, guiding the agent towards its goals effectively.
$R_i^{s}$ and $R_i^{a}$ represent the rewards associated with the SVO policy and the action policy, respectively. 
The hyper-parameter $\rho$ plays a crucial role in calibrating the influence of the SVO policy reward on the agent's learning process, allowing for fine-tuning to achieve desired behaviors. 
$R_i^{s}$ is conceptualized as the aggregate of the external rewards received by both the agent and its selected partner. 
This reward structure is rooted in the intention to encourage agents to adopt SVO choices that enhance the collective well-being of the small group (consisting of the agent and its partner). 
Through this tightly coupled mechanism, the framework incentivizes the agent to learn and select SVOs that are not only beneficial to itself, but also advantageous to its cooperative interactions, thereby facilitating coordinated maneuvers.
$R_i^{a}$, on the other hand, is based on the definition of SVO. This ensures that the agent's actions are coherent with its selected SVO ($Z$). 
The reward received through $R_i^{a}$ motivates the agent to execute actions that are in harmony with its SVO, fostering a congruent and integrated approach to decision-making and behavior. By simultaneously maximizing cumulative $R_i^{s}$ and $R_i^{a}$, the model can derive both a upper-level SVO policy and a lower-level action policy: 
\begin{equation}
\begin{aligned}
    \pi^*_\phi(z_i|z'_i) &= \arg\max_\phi\mathbb{E}_{\tau_i\thicksim\pi_\theta}[\sum_{t=0}^T\gamma^tR_i^{s}]; \\
    \pi^*_\theta(a_i|o_i,z'_i) &= \arg\max_\theta\mathbb{E}_{\tau_i\thicksim\pi_\theta,Z_i\thicksim\pi_\phi}[\sum_{t=0}^T\gamma^tR_i^{a}];
\end{aligned}
\end{equation}
where $\gamma\in (0,1]$ is the discount factor and $T$ is the time horizon.

\subsubsection{Policy optimization}

Proximal Policy Optimization (PPO) stands out as a highly favored framework in the domain of reinforcement learning (RL), celebrated for its stability, straightforward hyper-parameter tuning, and impressive performance. 
In our work, we use a variation of PPO to support the training of SYLPH agents while distinguishing this approach from vanilla PPO by integrating an additional layer: a higher-level, socially-aware policy. 
This novel layer complements the action policy, equipping agents with social behavior in addition to moving towards their goals. 
We call this novel variation as Social-aware Multi Policy PPO (SMP3O).

In the context of SMP3O, where the framework needs to optimize two policies (the SVO policy $\pi_\phi(z_i|z_i')$ and the action policy $\pi_\theta(a_i|o_i,z_i')$) simultaneously, it becomes important to consider the losses associated with each policy independently. 
We formalize the losses for these policies as $\mathcal{L}_{\pi_\phi}$ and $\mathcal{L}_{\pi_\theta}$:
\begin{equation}
\begin{aligned}
    \mathcal{L}_{\pi_\phi} &= \mathbb{E}_t[\min(r_{svo}^t(\phi)\hat{A}_{action}^t,clip(r_{svo}^t(\phi), 1-\epsilon, 1+\epsilon)\hat{A}_{action}^t)]\\
    \mathcal{L}_{\pi_\theta} &= \mathbb{E}_t[\min(r_{action}^t(\theta)\hat{A}_{svo}^t,clip(r_{action}^t(\theta), 1-\epsilon, 1+\epsilon)\hat{A}_{svo}^t)]
\end{aligned}
\end{equation}
where $r_{svo}^t(\phi) = \frac{\pi_\phi(z_i^t|z_i^{t-1})}{\pi_{\phi_{old}}(z_i^t|z_i^{t-1})}$ and $r_{action}^t(\theta) = \frac{\pi_\theta(a_i^t|o_i^t,z_i^t)}{\pi_{\theta_{old}}(a_i^t|o_i^t,z_i^t)}$ are the probability ratios of the action under the new policy $\pi_{\phi/\theta}$ over the old policy $\pi_{\phi_{old}/\theta_{old}}$.
$\epsilon$ is a hyper-parameter that defines the clipping range to avoid excessively large policy updates.
$\hat{A}^t_{svo}$ and $\hat{A}^t_{action}$ are the advantage functions, which estimate how much better a particular SVO/action is if it was taken over the average.
Unlike the policy loss in PPO, which is calculated by directly associating the advantage of an action with the likelihood ratio of that action under the current policy versus the old policy, SMP3O introduces a novel approach.
The core insight behind integrating the SVO policy with the action policy in the SMP3O framework lies in leveraging the hierarchical nature of these policies. 
Specifically, we provide a cross-utilizing advantages mechanism, where each policy derives indirect benefits from the learning signals of the others.
As the upper level, the SVO policy plays a pivotal role in redistributing action rewards, thereby ensuring that actions are aligned with the social preferences of the agent (as illustrated in Eq~\ref{reward_redistri}). 

Therefore, when calculating the action policy loss $\mathcal{L}_{\pi_\theta}$, $\hat{A}^t_{svo}$ should be combined with the action policy $\pi_\theta(a_i|o_i,z_i')$. 
This mechanism is particularly beneficial in scenarios characterized by social dilemmas. 
In such situations, even though the SVO policy may dictate a course of action that aligns with long-term group benefits or conflict resolution, it might disadvantageous to the agent in the short term. 
SMP3O encourages the action policy to transcend this myopia, adhering the directives of the SVO policy, which can be expressed as:
\begin{equation}
\begin{aligned}
    \triangle\theta\propto\nabla_\theta\mathbb{E}[\hat{A}_{svo}\cdot\log\pi_\theta(a_i|o_i,z_i')]
\end{aligned}
\end{equation}
This approach fosters behaviors that potentially sacrifice immediate individual gains and instead contribute to the collective well-being of the agent pair, emphasizing group rewards over individual rewards.
When calculating the SVO policy loss $\mathcal{L}_{\pi_\phi}$, incorporating the action advantage $\hat{A}_{action}$ facilitates the formation of a feedback loop. 
Because of reward reallocation and action loss, agent action will be consistent with the SVO decision.
Under this premise, the aforementioned feedback enables the SVO policy to assess the quality of executed actions, guiding the agent towards decisions that concurrently optimize both SVO alignment and action efficacy.
Formally:
\begin{equation}
\begin{aligned}
    \triangle\phi\propto\nabla_\theta\mathbb{E}[\hat{A}_{action}\cdot\log\pi_\phi(z_i|z_i')]
\end{aligned}
\end{equation}
This bidirectional reinforcement between the SVO and action policies create a synergistic learning environment. 
It not only aligns individual actions with broader social goals but also ensures that the SVO policy is refined based on the outcomes of these actions, establishing a coherent and mutually beneficial relationship between social behaviors and action executions.

To enhance the stability of an agent's SVO, we rely on supervised learning. 
Drawing from our prior experiences with valid and blocking losses, the formulation for the SVO stability-enhancing loss, $\mathcal{L}_{stab}$, can be expressed as follows:
\begin{equation}
\begin{aligned}
    \mathcal{L}_{stab} = \mathbb{E}[z_{i,exp}\log(\pi_\phi(z_i|z_i'))+(1-z_{i,exp})\log(1-\pi_\phi(z_i|z_i'))]
\end{aligned}
\end{equation}
Here, $\mathcal{L}_{stab}$ is fundamentally the cross entropy between the SVO policy output by the neural network and the expected SVO probability distribution. 
For shaping our desired SVO distribution, we introduce an under-relaxation factor, $\alpha$, within the range [0,1], to modulate the adjustment of the SVO towards the expected value. 
The expected SVO, $z_{exp}$, is determined by blending the current SVO, $z$, with previous SVO, $z'$, using $\alpha$:
\begin{equation}
\begin{aligned}
    z_{i,exp} = \alpha z_i' + (1-\alpha) z_i
\end{aligned}
\end{equation}
The calculation of $\alpha$ is designed to reflect the degree of overlap between the agent and its partner:
\begin{equation}
\begin{aligned}
    \alpha = \min(\mathcal{O}_L[i, p], clip(\mathcal{O}_L[i, p], 0, \kappa)) / \kappa
\end{aligned}
\end{equation}
$\kappa$ sets a boundary condition for the overlap magnitude; when the actual overlap between agents falls below $\kappa$, the agent is permitted to modify the SVO with greater latitude.
The rationale behind this formulation is to encourage SVO policy stability especially under conditions of significant overlap, thereby mitigating potential volatility in the agent's social behavior. 
As the overlap diminishes, indicating reduced immediate conflict or competition for resources, the model permits more flexible adjustments to the SVO, aiming to foster higher degrees of cooperation and coordination.

\subsection{Enhancing Policy with SVO}
\label{Enhancing_policy}
 
\begin{figure}[h]
	\centering
	\vspace{-0.3cm}
	\subfigtopskip=2pt 
	\subfigbottomskip=2pt 
	\subfigcapskip=-5pt
	\subfigure[]{
		\label{tie_breaking_a}
		\includegraphics[width=0.2\linewidth]{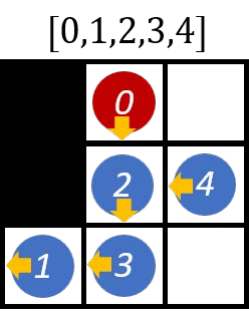}}
	\quad 
	\subfigure[]{
		\label{tie_breaking_b}
		\includegraphics[width=0.2\linewidth]{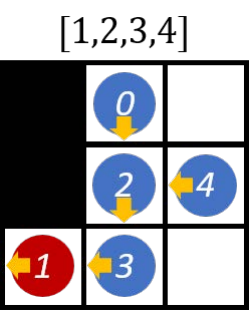}}
    \quad
	\subfigure[]{
		\label{tie_breaking_c}
		\includegraphics[width=0.2\linewidth]{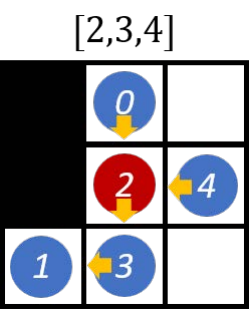}}
	\quad
	\subfigure[]{
		\label{tie_breaking_d}
		\includegraphics[width=0.2\linewidth]{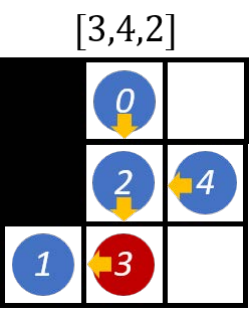}}
  
    \subfigure[]{
		\label{tie_breaking_e}
		\includegraphics[width=0.2\linewidth]{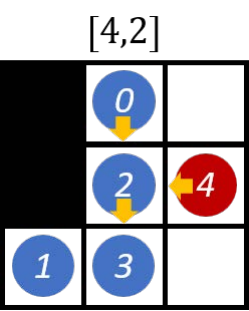}}
	\quad 
	\subfigure[]{
		\label{tie_breaking_f}
		\includegraphics[width=0.2\linewidth]{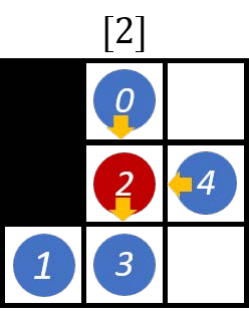}}
    \quad
	\subfigure[]{
		\label{tie_breaking_g}
		\includegraphics[width=0.2\linewidth]{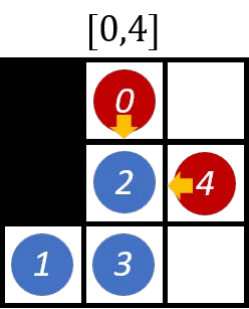}}
	\quad
	\subfigure[]{
		\label{tie_breaking_h}
		\includegraphics[width=0.2\linewidth]{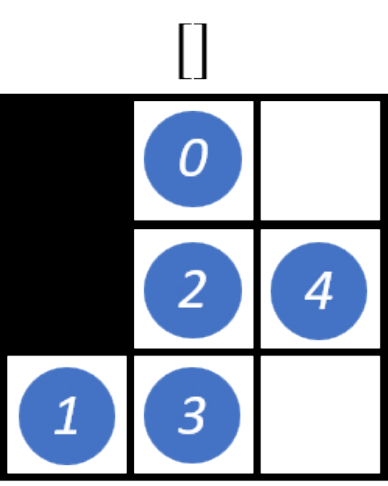}}
	\caption{The SVO-based tie-breaking mechanism example.}
	\label{tie_breaking}
\end{figure}

Our previous work, PRIMAL~\cite{sartoretti2019primal}, established a definitive set of valid actions and aimed to train agents to recognize meaningful/executable actions through supervised learning.
However, we observed that PRIMAL agents, motivated by the pursuit of higher individual rewards, would still sometimes select invalid actions.
This tendency adversely impacted the overall system's performance. 
To mitigate such issues,~\cite{wang2023scrimp} introduced a tie-breaking strategy based on state value, essentially serving as a rule-based local post-processing. 
This strategy empowered agents to evaluate all possible scenarios in the next timestep, enabling the selection of a more reasonable action.
This rule-based method proved computationally intensive especially in densely populated environments, escalating the computational burden significantly. 
In response to this challenge, this paper proposes an innovative approach that integrates SVO into the conflict resolution process, redefining reward distribution amidst conflicts. 
This strategy departs from post-processing by embedding a tie-breaking mechanism directly within the model, thereby enhancing agent's policy without incurring additional computational costs.
This RL-based tie-breaking mechanism, unlike its predecessor, does not merely rectify decisions post-facto but rather informs the decision-making process intrinsically, ensuring that choices made are both feasible and aligned with the collective objective. 

\begin{algorithm}[t]
	\caption{SVO-based Tie-breaking Method.}
	\label{algo_3}
	\KwIn{All agent's action set $\mathcal{A}\in\mathbb{N}^n$ and SVO set $\mathcal{Z}\in\mathbb{N}^n$ output by the neural network.}
	\KwOut{Adjusted conflict-free action set $\mathcal{A}'\in\mathbb{N}^n$; Socially-aware adjusted rewards $R\in\mathbb{R}^n$.}
	\BlankLine

         \textbf{Sort} agents in descending order of SVOs $\mathcal{Z}$ to obtain the consideration chain $\mathcal{C}$;
         \hfill \textcolor{blue}{$\triangleright$ \texttt{initialize consideration chain}}
         
         \While{$len(\mathcal{C}) < 0$}{
         $i \gets$ $\mathcal{C}$.pop() \hfill \textcolor{blue}{$\triangleright$ \texttt{check action status}}
         
         \uIf{$\mathcal{A}[i]\in$ Invalid Action}{
         $\mathcal{A}'[i] = 0$ (stay idle); $R[i] = -2$ (Collision Penalty);
         \hfill \textcolor{blue}{$\triangleright$ \texttt{invalid}}
         \\
         \uIf{$0\in$ Restricted Action}{
         \For{$\forall j~\mathbf{causing}$ \textit{Restricted Action}}{
         $\mathcal{C}$.append(j) $\mathbf{If}~j\notin\mathcal{C}$ 
         }}
         }
         \uElseIf{$\mathcal{A}[i]\in$ Restricted Action}
         {
         \For{$\forall j~\mathbf{causing}$ \textit{Restricted Action}}{
         $\mathcal{A'}[i]=0$ (stay idle); 
         \hfill \textcolor{blue}{$\triangleright$ \texttt{restricted}}
         \\
         $R[i]=-2~\mathbf{If}~\mathcal{Z}[i]>\mathcal{Z}[j]$;
         $R[j]=-2~\mathbf{If}~\mathcal{Z}[j]>\mathcal{Z}[i]$;
         }
         $\mathbf{Repeat}$ Line 6 to 9
         }
         \Else{
         $\mathcal{A}'[i]=\mathcal{A}[i]$; $R[i]=$ External Reward;
         \hfill \textcolor{blue}{$\triangleright$ \texttt{normal}}
         }
         }
\end{algorithm}

The tie-breaking mechanism in our paper outlines an approach to conflict resolution among agents based on their social preferences.
We first need to clarify the definitions of \textit{Invalid Actions} and \textit{Restricted Actions}.
Invalid Actions are actions that would cause an agent to collide with a static object or boundary within the environment, rendering the action unfeasible.
Restricted Actions are actions would result in dynamic collisions between agents, depending on their intended movements during the same timestep.
To process action validation and conflict resolution, agents are sorted based on their SVOs at the begin, from the most prosocial to the most self-interested [Line 1]. 
This ranking dictates the priority with which each agent's actions are validated and adjusted in the face of potential conflicts.
Then, agents are checked in descending order of prosociality for invalid actions [Line 2-9].
If an invalid action is detected, the agent's action is changed to 'stay idle' to avoid static collision [Line 5].
If 'staying idle' still results in a restricted action (potential collision with another agent), this action is flagged, and the situation needs further resolution [Line 6-9].
The next operation of Algorithm~\ref{algo_3} is resolving restricted actions.
For agents causing dynamic collisions, their actions are set to 'stay idle' [Line 12].
If waiting does not resolve the collision, actions of conflicting agents are reassessed and adjusted at the end of the consideration chain.
The chain continues until all agents have valid actions [Line 15].
Only agents with higher SVOs (more prosocial) receive penalties for collisions [Line 13], reflecting their role in facilitating smoother group dynamics by yielding. 
Self-interested agents are less likely to be penalized, preserving their direct routes or actions unless absolutely necessary.
This tie-breaking mechanism effectively uses SVO as a prioritization tool in conflict resolution, where more cooperative agents are more inclined to compromise for the greater good. 
The algorithm potentially make these systems more acceptable and understandable in human-centric environments.

Fig.~\ref{tie_breaking} illustrates an example scenario. 
The numbers in circles indicate the initial order in which agents are checked based on their SVOs, arranged in descending order. 
The top array represents the current chain of consideration for resolving the collision, initially set as $[0, 1, 2, 3, 4]$.
In Fig.~\ref{tie_breaking_a}, agent 0’s action is assessed first. 
It appears valid and remains unchanged, assuming that if agent 2 can execute its current plan, there will be no conflict. 
Next, in Fig.~\ref{tie_breaking_b}, agent 1 is found to be performing an invalid action and is therefore set to stay idle. 
However, this results in a restricted action affecting agent 3 (Fig.~\ref{tie_breaking_d}), prompting its addition to the list for evaluation.
Before this, Fig.~\ref{tie_breaking_c} shows that agent 2 is performing a valid action.
When agent 3's action is evaluated, it is changed to 'stay idle,' necessitating a reevaluation of agent 2 after agent 4’s actions are considered (as shown in Fig.~\ref{tie_breaking_e}). 
Upon reconsideration, agent 2’s action is deemed restricted (Fig.~\ref{tie_breaking_f}), and it is set to stay idle. 
This change prompts the addition of agents 0 and 4 back into the consideration chain, as shown in Fig.~\ref{tie_breaking_g}.
The deadlock is ultimately resolved when all agents are set to 'stay idle.' 
In this scenario, agents 1, 2, and 0 receive collision penalties due to conflicts with agents 3, 4, and 2 respectively.


\section{Attention-based Network}

The network architecture described encompasses three distinct input components as shown in Fig.~\ref{net_overview}. 
The first channel is grid-level observation, essentially the detailed environment within the agent's field of view (FoV), centered around itself. 
The second channel is a vector that directs the agent towards its goal. 
The third channel comprises the SVOs of the agent and its partners, which are selected according to Section~\ref{partner_selection}. 
For processing the grid-level observations within the FoV and the vector pointing towards the goal, we utilize the efficient encoder as outlined in our prior research~\cite{sartoretti2019primal}. 
Specifically, the grid-wise observation undergoes processing through two 2D convolutional blocks~\cite{simonyan2014very}, each containing two convolutional layers followed by a MaxPooling operation. 
The encoding of the goal vector is accomplished using a fully connected layer.
Meanwhile, the SVOs of the ego agent and its partners are encoded using a communication block that is based on a Graph Transformer~\cite{chen2022nagphormer}.
Subsequently, the embeddings generated from these three inputs are concatenated and fed through a residual block before being decoded by a semantic transformer.
The outputs are then transformed into the agent's action and SVO policies, as well as blocking and value, through different fully connected layers and activation functions. 
This section will delve into the detailed implementation of the communication block and semantic transformer block. 

\begin{figure}[t]
\centering
\includegraphics[width=5.4in]{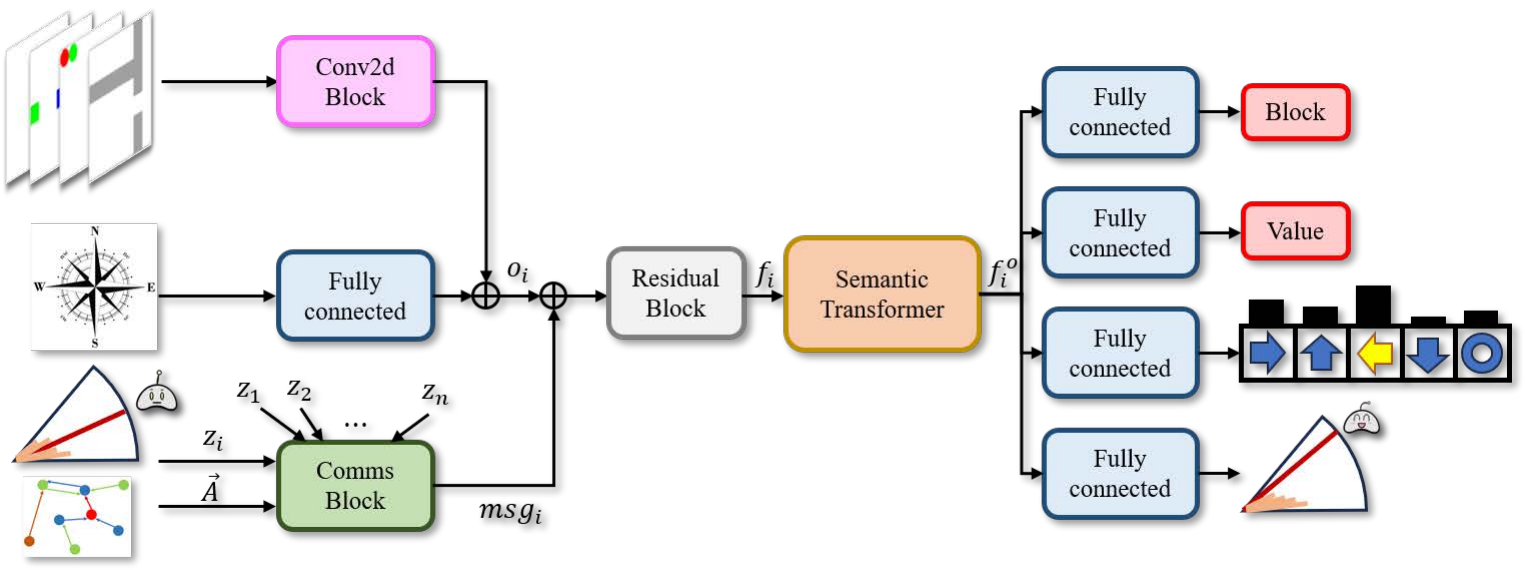}
\vspace{-0.3cm}
\caption{Overview of the Network of SYLPH.}
\label{net_overview}
\vspace{-0.4cm}
\end{figure}


\subsection{Communication Block}

\begin{figure}[t]
\centering
\includegraphics[width=5.4in]{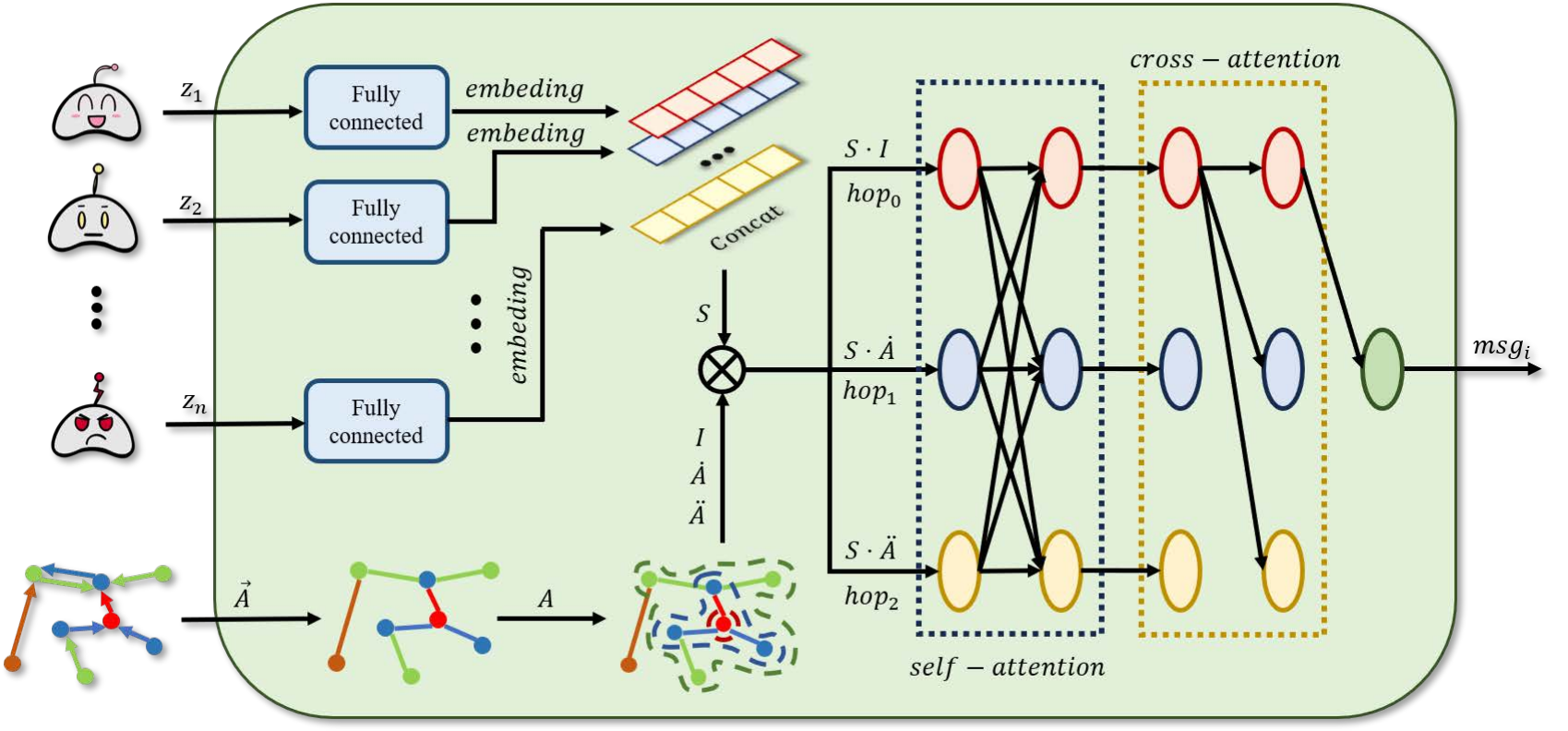}
\vspace{-0.3cm}
\caption{Details of communication block.}
\label{comms_block}
\vspace{-0.4cm}
\end{figure}

Our communication block draws inspiration from the Hop2Token aggregation approach detailed in~\cite{chen2022nagphormer}. 
This approach focuses on aggregating vertices within a graph into "hops" based on the graph's structure, acknowledging that these hops possess varying degrees of significance to the ego agent. 
This perspective aligns perfectly with our scenario, where agents are represented as vertices and their selected partners as edges in a graph. 
Notably, even distant agents can become adjacent nodes if their influence on the ego agent is substantial.
Therefore, combined with the partner selection Algorithm~\ref{algo_1} of Section~\ref{partner_selection}, hop-based communication is more suitable for our framework than general distance-based communication.

In the communication block illustrated in Fig.~\ref{comms_block}, the inputs include the global agent's social preference SVO $z_i$ and the adjacency matrix $\vec{A}$ of the directed graph $\mathcal{G}_a$, which is constructed according to the partner selection mechanism. 
The first step involves converting $\vec{A}$, the directed adjacency matrix, into $A$, its undirected counterpart. 
Subsequent operations involve raising $A$ to the $0$th, $1$st, and $2$nd powers and symmetric normalization ($SN$), yielding the identity matrices $I$, $\dot{A}$, and $\ddot{A}$ as multipliers for multi-hops. 
Concurrently, the SVO probability distribution $z_i$ is transformed into SVO embeddings via fully connected layers, serving as vertex features $v_i$.
These vertex features are then combined with multi-hop multipliers to encapsulate the graph's entire feature into a condensed number of multi-hop nodes $k$ ($k=2$ in practice), significantly streamlining computations ($k\ll n$). 
Following this feature aggregation, the multi-hop nodes $h_i$ are processed through multi-head self-attention and cross-attention layers. 
The output features of $0$-hop node $h_0$, emerging from these attention mechanisms, are utilized as the encoder embedding $msg_i$ for the SVO channel, effectively capturing the interactions and dependencies within the graph from agent $i$'s perspective.
Formally:
\begin{equation}
\text{Embedding}\left\{
\begin{aligned}
    &A = \vec{A} \vee \vec{A^T} \\
    &I, \dot{A}, \ddot{A} = SN(A^0, A^1, A^2) \\
    &S = Concat(FF(z_1), FF(z_2), \cdots, FF(z_n)) \\
    &H = Concat(h_0,h_1,h_2) = S\otimes I, S\otimes\dot{A}, S\otimes\ddot{A}
\end{aligned}
\right.
\end{equation}
\begin{equation}\label{eq_self_att}
\text{Self-attention}\left\{
\begin{aligned}
    &Q^h_s, K^h_s, V^h_s = W^h_{Q, s}H, W^h_{K, s}H, W^h_{V, s}H \\
    &A^h_s = Att(Q^h_s, K^h_s, V^h_s) = Softmax(\frac{Q^h_s\cdot {K^h_s}^T}{\sqrt{d_k}})\cdot V^h_s \\
    &\hat{h}_s = BN(H + Concat(A^1_s, A^2_s, \cdots, A^H_s)W_{O,s}) \\
    &H_s^o = BN(\hat{h}_s+FF(\hat{h}_s))
\end{aligned}
\right.
\end{equation}
\begin{equation}
\text{Cross-attention}\left\{
\begin{aligned}
    &Q^h_c, K^h_c, V^h_c = W^h_{Q, c}h_0, W^h_{K,c}H_s^o, W^h_{V, c}H_s^o \\
    &A^h_c = Att(Q^h_c, K^h_c, V^h_c) = Softmax(\frac{Q^h_c\cdot {K^h_c}^T}{\sqrt{d_k}})\cdot V^h_c \\
    &\hat{h}_c = BN(h_0 + Concat(A^h_c, A^h_c, \cdots, A^H_c)W_{O,c}) \\
    &h_c^o = BN(\hat{h}_c+FF(\hat{h}_c))
\end{aligned}
\right.
\end{equation}
where $FF(\cdot)$ means the fully connected layer; $Concat(\cdot,\cdot)$ indicates the concatenate operator; $BN(\cdot)$ denotes batch normalization.


\subsection{Semantic Transformer Block}

\begin{figure}
\centering
\includegraphics[width=5.4in]{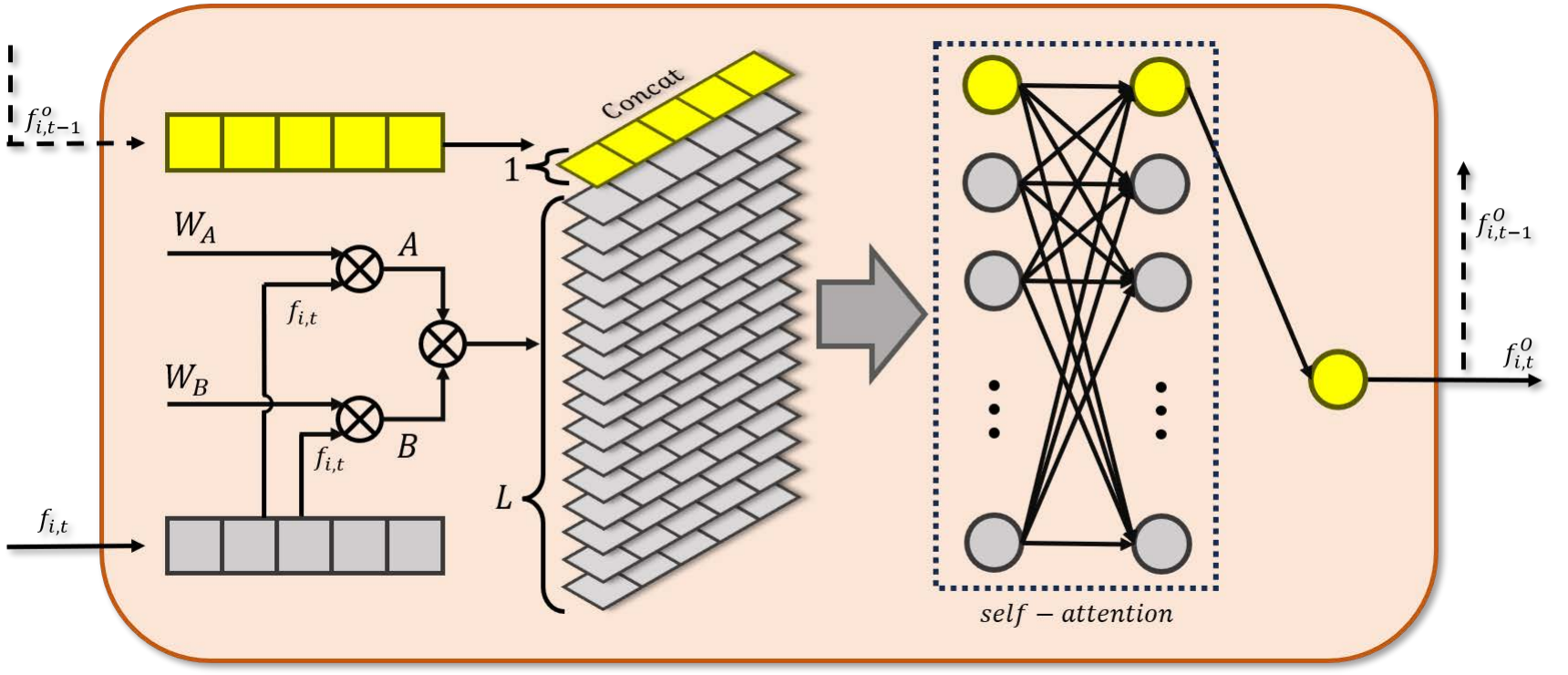}
\vspace{-0.3cm}
\caption{Details of semantic transformer block.}
\label{semantic_block}
\vspace{-0.4cm}
\end{figure}

In contrast to earlier efforts~\cite{virmani2021subdimensional} that applied a Visual Transformer as an encoder to enhance embeddings for observations with a larger FoV, our approach uniquely integrates a Semantic Transformer into our existing Multi-Agent Path Finding (MAPF) framework, supplanting the LSTM component. 
This transformer serves as a decoder, analyzing and restructuring low-level features from diverse observational channels related to the agent, such as FoV observations, the directional vector to its goal, and the SVOs of agents. 
Inspired by the tokenizer approach used in image semantic segmentation as discussed in~\cite{wu2020visual}, we suggest that the agent's environmental state can also be captured in a similar way.
By collecting and reorganizing low-level features, we aim to distill these into a predetermined number of semantic tokens ($L=16$ in practice) via a spatial attention mechanism. 
These tokens, embodying high-level semantic concepts, articulate the agent's environmental state context and are subsequently processed by a standard transformer~\cite{vaswani2017attention} for contextual reasoning.
Notably, echoing the temporal reliance inherent in a long-short-term memory (LSTM) cell, our model aspires to imbue the semantic reasoning process with a memory / recurrent mechanism. 
As illustrated in Fig.~\ref{semantic_block}, an additional memory token, representing the output from the previous timestep's Semantic Transformer module, is incorporated. 
The cross-attention outcomes between this memory token and the current timestep's generated semantic tokens are harnessed as the present output of the Semantic Transformer module. 
This refined output informs the generation of the agent's diverse policy actions, melding past insights with current observations to navigate the agent in a clever way. 
It can be expressed mathematically as follows:
\begin{equation}
\begin{aligned}
    S_T = Concat(f^O_{i,t-1},((f_{i,t}\cdot W_A)\otimes(f_{i,t}\cdot W_B))).
\end{aligned}
\end{equation}
$S_T$ represents a set of tokens that includes $L$ semantic tokens alongside a memory token. 
It is then fed into a standard Transformer architecture, where it undergoes an update process facilitated by the self-attention mechanism (as shown in Eq~\ref{eq_self_att}) inherent to Transformers. 
The self-attention mechanism enables the model to dynamically weigh the importance of each token within $S_T$ relative to the others, thereby refining their representations based on the contextual relationships within the set.
Following the processing through the Transformer, the memory token is specifically extracted from the updated set $S_T$. 
This memory token serves a dual purpose: it not only represents the output of the semantic transformer block for the current timestep, but also becomes the input for the semantic transformer block in the subsequent timestep. 
This cyclical integration of the memory token allows for the preservation and transference of context and learned information across timesteps, effectively enabling the model to maintain a continuous thread of relevant information throughout the sequence of actions or decisions.
The incorporation of a memory token within the set of semantic tokens facilitates a more nuanced and contextually aware processing of information, enhancing the model's ability to make informed decisions based on both the current semantic context and the historical context encapsulated within the memory token.


\section{Experiments}

In this section, we give some simulations and experimental validation for the proposed framework and mechanism.

\subsection{Symmetric Pathfinding Case Study}
\label{sym_exp}

In our study, we first consider experiments in a completely symmetric environment, a setting inherently susceptible to social dilemmas. 
Using the configuration employed in~\cite{bettini2023heterogeneous} (as depicted in Fig.~\ref{exp1}(a)), we placed two agents at opposing ends of a narrow corridor which are wide enough to accommodate just one robot at a time, with each agent's starting position serving as the other's goal.

\begin{figure}[h]
\centering
\includegraphics[width=4in]{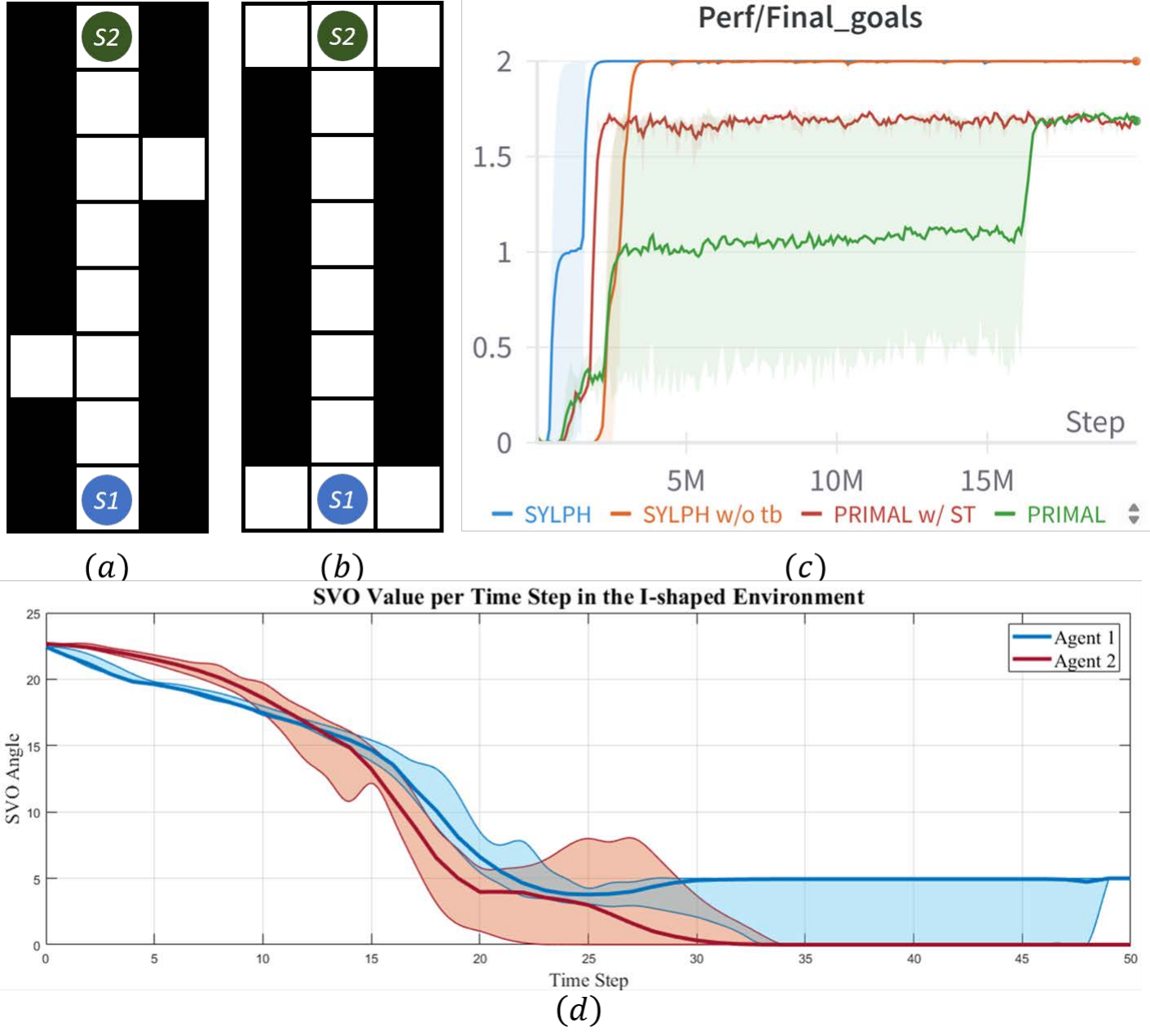}
\vspace{-0.3cm}
\caption{Two agents symmetric pathfinding case study.}
\label{exp1}
\vspace{-0.4cm}
\end{figure}

The first type of corridors have two recesses capable of fitting one robot each, making the pathfinding environment solvable.
These recesses are symmetrically positioned, yet unlike~\cite{bettini2023heterogeneous}, the corridor length and the recesses' locations are subject to random variation in our experiments.
Further diversifying our experimental setup, we introduced an \texttt{I}-shaped map, presented in Fig.~\ref{exp1}(b). 
Similar to the first environment, the narrow corridor's length within this map remains variable. 
Throughout our training process, we maintained a probabilistic distribution for the occurrence of each map type: $80~\%$ ($p_r=80~\%$) for the corridor with recesses (recess map) and $20~\%$ ($p_I=20~\%$) for the \texttt{I}-shaped map.

To deduce the expected number of goals reached by PRIMAL\footnote{All references to PRIMAL in this section are modified versions, not the original one. 
There are two primary differences: firstly, the model is trained using the Proximal Policy Optimization (PPO) framework rather than the Asynchronous Advantage Actor-Critic (A3C); secondly, the model's FoV incorporates heuristic maps designed in DHC.} and SYLPH under the given map probabilities, we noticed the training results of PRIMAL reaches approximately 1.7 goals on average, while SYLPH achieves 2 goals, as shown in Fig.~\ref{exp1}(c). 
Given the probabilities of encountering each map type, $p_r=80~\%$ for the recess map and $p_I=20~\%$ for the \texttt{I}-shaped map, we can calculate the expected contributions from each map type to the overall performance metrics:
\begin{equation}
    \begin{aligned}
        & \mathbb{E}(g_{PRIMAL}) = p_{r} \times 2 +p_{I} \times 1  = 1.7\\
        & \mathbb{E}(g_{SYLPH}) = p_{r} \times 2 + p_{I} \times 2 = 2
    \end{aligned}
\end{equation}
PRIMAL's difficulty with the \texttt{I}-shaped map, despite its competence in navigating the recess map, underscores a fundamental limitation in its approach: an inclination towards self-preservation that precludes effective resolution of social dilemmas. 
This inclination towards self-interest, a characteristic strongly embedded in PRIMAL-trained models, hampers their ability to high level collaboration. 
As a result, they are often trapped in live-/deadlock instead of superior collective outcomes.
SYLPH, on the other hand, breaks this perfect symmetry and enhances coordinated maneuvers by introducing different social preferences/SVOs to agents. 
This enables some agents to make more prosocial/selfless choices in such a completely symmetric environment, promoting the achievement of team goals over individual optimization.
For instance, in the \texttt{I}-shaped map scenario, the change curves of SVO over time for different agents are shown in Fig.~\ref{exp1}(d). 
The diverse SVO values among the team effectively addresses the coordination challenges inherent in the \texttt{I}-shaped map, showcasing SYLPH's superior adaptability and its potential to resolve complex social dilemmas within symmetric environments.


\subsection{Comparative Experiments}

\subsubsection{Performance comparison}

In our comparison experiments, we primarily assessed three metrics: (1) the \textit{success rate} across 200 instances under a similar configuration, (2) the average \textit{episode length} required to complete these instances, and (3) the average ratio of agents arrive their goals (\textit{arrival rate}) per instance among the 200 instances.
The first metric directly assesses the effectiveness of the planner in achieving complete solutions across a wide range of scenarios. 
A high success rate indicates robustness and reliability, showing that the planner can consistently solve the MAPF problem under varying conditions.
The second one reflects the efficiency of the pathfinding algorithm and the quality of the solutions it generates.
The last metric is particularly revealing, as it accounts for the performance of individual agents within partially successful episodes. 
It provides a more detailed picture of a planner's performance, especially in scenarios where not all agents may reach their destinations due to complex interactions or partial failures.
For traditional algorithms, success rate and arrival rate often coincide because these methods typically either succeed fully (all agents reach their goals) or fail entirely (one or more agents do not finish). 
Thus, these metrics tend to reflect the same performance aspect.
The learning-based planners might allow for more flexibility in agent behavior, leading to situations where some agents succeed while others do not in the same instance. 
Therefore, assessing these methods requires a finer-grained metric like the arrival rate to capture the nuances of partial successes and individual agent failures, which the success rate alone might overlook.

\begin{figure}[h]
\centering
\includegraphics[width=5.2in]{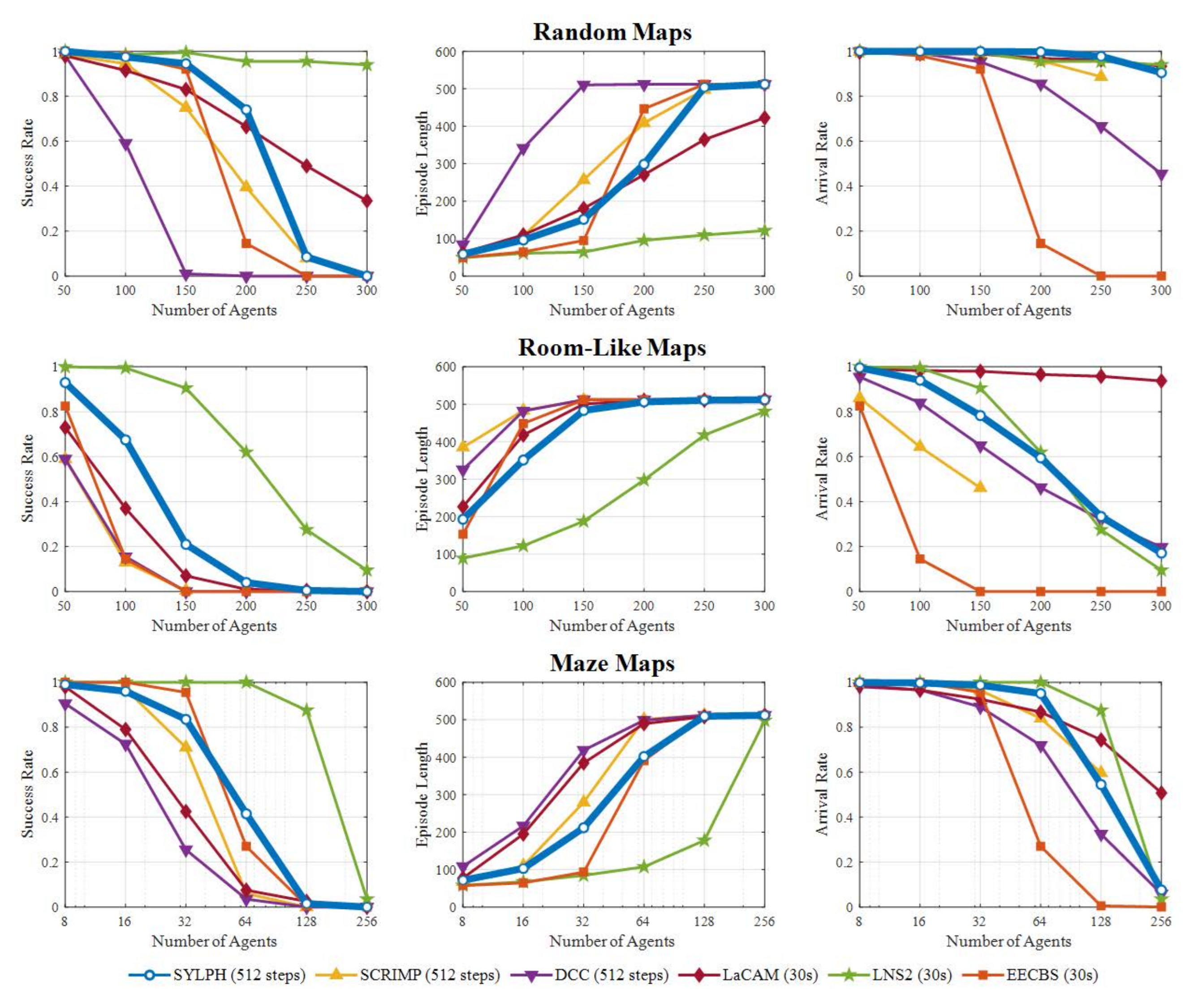}
\vspace{-0.3cm}
\caption{Comparative results of SYLPH and other baseline algorithms across various maps and configurations, analyzed using three performance indicators: success rate, episode length, and arrival rate.}
\label{compare}
\vspace{-0.4cm}
\end{figure}

To benchmark the performance of SYLPH against other state-of-the-art (SOTA) methodologies in MAPF field, we conducted a series of tests across a spectrum of extreme environment configurations. 
This rigorous testing was designed to evaluate how well SYLPH and its comparative baselines navigate increasingly complex scenarios, reflected in the type of the environments and the density of agent populations. Specifically, our test environment configuration is as follows:
\begin{itemize}
\item Random Maps: Utilized a $32\times 32$ grid world with an obstacle density of 0.2 (\textit{random-32-32-0.2}) and tested with varying numbers of agents: 50, 100, 150, 200, 250, and 300, to observe the scalability of the algorithms under different agent densities.
\item Room-like Maps: Similar in dimension and agent count to the random maps, these environments featured an increased obstacle density ($\approx 0.3$) and the orderly distribution of obstacles (\textit{room-32-32-0.3}), simulating more structured environments with additional coordination challenges.
\item Maze Maps: Given the heightened complexity and obstacle density of approximately 0.5 (\textit{maze-32-32-0.5}), the amount of agents was adjusted to smaller team size: 8, 16, 32, 64, 128, and 256, to test the algorithms' efficiency in highly constrained environments.
\end{itemize}

To provide a comprehensive comparison, we included both SOTA learning-based MAPF planners and traditional MAPF algorithms as baselines:
\begin{itemize}
\item SCRIMP~\cite{wang2023scrimp}: This is a learning-based MAPF planner recognized as SOTA. It achieves high performance across various environments by leveraging communication learning and a value-based tie-breaking mechanism. However, SCRIMP tends to be time-consuming in densely populated environments.
\item DCC~\cite{ma2021learning}: Another SOTA learning-based method, DCC, utilizes an attention mechanism to select communication partners during interactions, effectively reducing the communication load. While DCC is faster than SCRIMP, it does exhibit a lower success rate.
\item EECBS~\cite{li2021eecbs}: It represents a SOTA bounded suboptimal planner. Typically, setting the suboptimality factor as 1.2 strikes the best balance between path optimality and computational efficiency.
\item LNS2~\cite{li2022mapf}: This is currently one of the top MAPF planners. It resolves path conflicts through priority-based planning and continuous iteration, achieving near-optimal solutions with extremely high success rates.
\item LaCAM~\cite{okumura2023lacam}: 
Building on the PIBT solution, it searches all possible configurations to enhance planner performance. Unlike other traditional methods, LaCAM measures the arrival rate, providing an additional performance metric than just the success rate.
\end{itemize}
These baseline comparisons are intended to showcase SYLPH's adaptability and performance across different environmental complexities and agent densities. 
By testing SYLPH against both cutting-edge learning-based planners and established traditional algorithms, we aimed to highlight the framework's strengths, particularly in scenarios requiring advanced coordination and long-horizon coordination capabilities amidst varying degrees of environmental constraints.

\begin{table}[!t]
  \centering
  \caption{Time consumption of SYLPH and its baselines under different configurations.}
  \vspace{0.1cm}
  \scalebox{0.8}{
    \input{time_consuming}
  }
  \label{time_tab}
\end{table}

In Fig.~\ref{compare}, the comparative performance of SYLPH across various map configurations, including random, room-like, and maze maps, demonstrates its superior capability over other learning-based planners, including DCC and SCRIMP. 
The success of SYLPH across these metrics – reliability, scalability, and performance in structured environments – highlights its effective resolution of symmetry conflicts, a common challenge in highly structured scenarios such as room-like and maze maps.
This performance advantage is attributed to the diversity of social preferences within the agent population, enabling SYLPH to navigate complex interactions more effectively.
Additionally, SYLPH's design avoids the time-consuming post-processing phases that other learning-based methods rely on. 
Instead, it enhances model performance directly through its training and execution phases, which not only yields higher performance but also reduces runtime significantly compared to its counterparts, as shown in Table~\ref{time_tab}\footnote{The superscript $\blacktriangle$ indicates that the time is $C$++ time, otherwise it is $Python$ time.}.

When compared with the SOTA bounded suboptimal MAPF planner, EECBS, SYLPH exhibits distinctly better performance in both random and room-like maps, and comparable results in maze maps. 
Against LaCAM, particularly in more densely populated random maps, LaCAM shows a higher success rate, yet SYLPH maintains competitive individual agent arrival rates. 
In more structured environments, SYLPH generally surpasses LaCAM in success rates, although LaCAM may achieve better arrival rates.
Against LNS2, widely regarded as the most advanced MAPF planner for its reliability and optimality, SYLPH and all other planners can not get comparable results. 

These results confirm that SYLPH not only rivals but sometimes surpasses the performance of traditional SOTA MAPF algorithms in highly structured environments. 
This is a significant achievement for a learning-based framework, demonstrating that SYLPH's socially-aware approach effectively enhances its applicability and effectiveness across diverse and challenging MAPF scenarios. 
This experimental evidence solidifies SYLPH's position as a new SOTA method in learning-based MAPF, capable of delivering high-performance outcomes where other learning-based methods may struggle.

\subsubsection{Paired t-test}

In order to rigorously assess the effectiveness of the SVO policy implemented in SYLPH, we conducted a series of paired t-tests to statistically analyze the performance differences between SYLPH and configurations with randomly assigned SVOs. 
This study was structured to compare SYLPH against three random SVO assignment methods:
\begin{itemize}
    \item Random SVO Assignment at Each Step: This method, where an agent's SVO is randomly reassigned at every step, resulted in confusion and ineffective decision-making, as the frequent changes prevented the agents from leveraging any consistent strategy, leading to non-convergence.
    \item Static SVO Assignment for Each Episode: Assigning SVOs at the start of each episode and maintaining them throughout resulted in better stability and some level of task accomplishment.
    \item SVO Update Frequency Matching Partner Switching: Aligning the frequency of SVO updates with the frequency of switching partners shown similar performance as static SVO assignment for each episode.
\end{itemize}

The experiments were conducted across different map configurations (e.g. \textit{random-32-32-0.2-200}, \textit{room-32-32-0.3-100}, and \textit{maze-32-32-0.5-32}) with 8 agents\footnote{The training curves can be found at~\ref{random_svo}}, and the performance was statistically analyzed over 200 instances for each configuration. 
Specifically, for the random SVO experiments, to prevent any unfairness due to randomness, we conducted the experiment 10 times and used the average result.
The results from these tests, as detailed in Table~\ref{paired_t_test}, revealed significant statistical differences between the performances of SYLPH and the random SVO methods. 
Specifically, the p-values \textit{p} obtained from the paired t-tests were much lower than the conventional significance threshold (0.05, 0.01, or 0.001), indicating a statistically significant difference in performance favoring SYLPH.

\begin{table}[!t]
  \centering
  \caption{Paired t-test results.}
  \vspace{0.1cm}
  \scalebox{0.8}{
    \input{paired_t_test}
  }
  \label{paired_t_test}
\end{table}

The findings clearly demonstrate that while randomly assigned SVOs might achieve task completion in less structured and sparser scenarios, they lack the scalability and robustness required for handling complex, highly structured environments. 
Thus, the study conclusively validates the effectiveness of the learned SVO policy within the SYLPH framework, highlighting its critical role in advancing the capabilities of learning-based MAPF solutions.


\subsection{MAPF Ablation Experiments}

\begin{figure}
\centering
\includegraphics[width=5.4in]{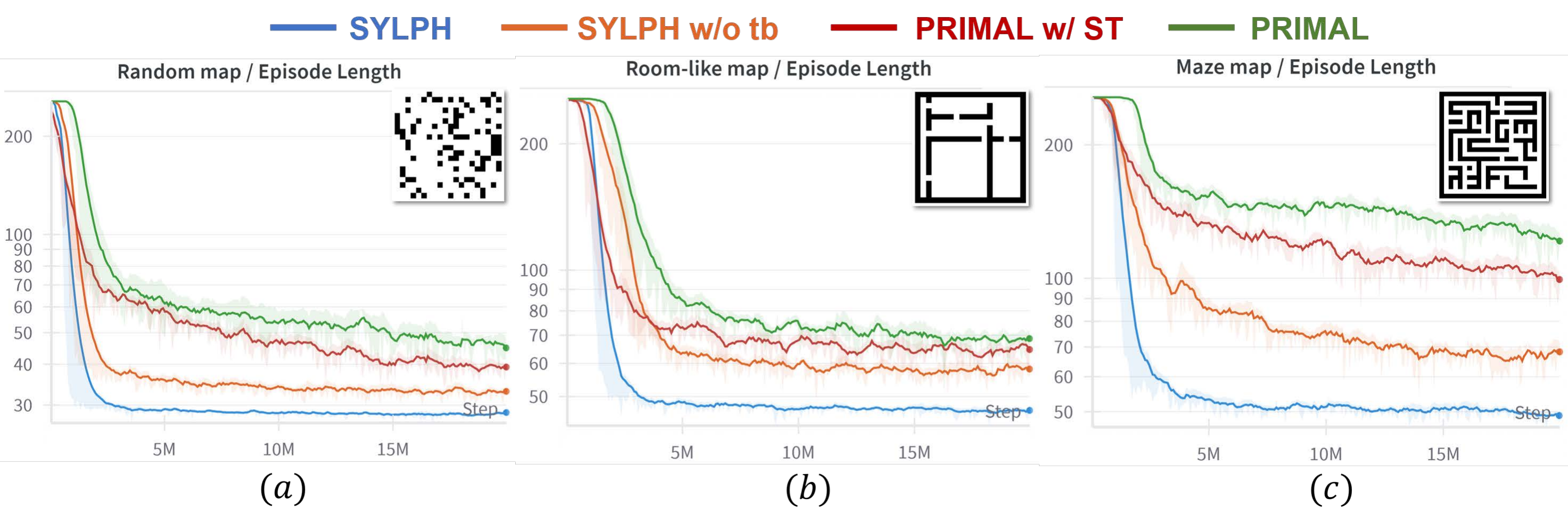}
\vspace{-0.3cm}
\caption{$(a)$, $(b)$, and $(c)$ represent the ablation experimental results of SYLPH variants in random maps, room-like maps, and maze maps, respectively. 
Consistent with Section~\ref{sym_exp}, PRIMAL here refers to a variant of PRIMAL that adds a heuristic maps.}
\label{ablation_el}
\vspace{-0.4cm}
\end{figure}

\textsc{All ablation studies took place in environments whose size is a random number sampled from $(10, 40)$ with 8 agents. 
A uniform maximum time limit was set for each episode, capped at 256 steps. 
All models are trained to 20M steps. }

The core idea of our architecture lies in integrating SVO as a temporally extended skill within the MAPF framework. 
This integration empowers agents with diverse social preferences, equipping them to navigate and resolve the symmetrical challenges typically presented by social dilemmas. 
To further enhance our framework, we replaced the LSTM component with a Semantic Transformer (ST), thereby augmenting the agent's spatial inference capabilities. 
Additionally, we developed a learning-based tie-breaking mechanism, offering a more efficient planner in densely populated environments compared to general post-processing approaches.

To evaluate the impact of these key elements, we tested four SYLPH variants across different map types (random, room-like, and maze maps): 1) The complete SYLPH model incorporating all mentioned components; 2) SYLPH minus the learning-based tie-breaking mechanism (SYLPH w/o tb), retaining SVO as a temporally extended skill; 3) SYLPH stripped of all SVO-related components, including partner selection, SMP3O, and communication block; 4) based on (3), the semantic transformer replaced by LSTM.
The performance enhancements brought by each element were quantified through experiments, with results showcased in Fig.~\ref{ablation_el}. 
We utilized Episode Length - the timesteps required for all agents to achieve their goals - as the performance metric. 
Notably, across the different map types, incremental additions of SYLPH components yielded significant performance improvements:
\begin{itemize}
    \item Random Map: Transforming from PRIMAL to PRIMAL w/ Semantic Transformer (ST) reduced episode length by $13.86~\%$. Incorporating SVO decreased episode costs by $29.67~\%$ compared with PRIMAL. After the inclusion of the tie-breaking mechanism, SYLPH's episode length is $38.93~\%$ lower than PRIMAL.

    \item Room-like Map: Here, the episode length saw a reduction of $4.71~\%$ with PRIMAL w/ ST, $15.93~\%$ with SYLPH w/o tb, and $33.02~\%$ with the full SYLPH model, all compared to the baseline PRIMAL performance.

    \item Maze Map: This environment highlighted the most pronounced improvements: $16.35~\%$ (PRIMAL w/ ST), $43.13~\%$ (SYLPH w/o tb), and $58.88~\%$ (SYLPH), showcasing the framework's efficacy in addressing social dilemmas, particularly prevalent due to the maze's long corridors. 
\end{itemize}
These findings underscore the substantial benefits of each component, especially in complex environments like maze maps where social dilemmas always appear. 
The progression from PRIMAL to the fully-fledged SYLPH model demonstrates the significant role of spatial inference enhancement, SVO diversification, and the learning-based tie-breaking mechanism in elevating model performance across varied and challenging MAPF scenarios.

\begin{table}[!t]
  \centering
  \caption{More Performance of Ablation Study.}
  \scalebox{0.8}{
    \input{tab_ablation}
  }
  \label{ablation_tab}
\end{table}

More other performance changes are shown in Table~\ref{ablation_tab}.
Specifically, we introduce three more metrics: external reward, blocking times, and goals reached. 
These metrics collectively offer a comprehensive view of each system's efficiency and cooperative capabilities under the constraints of the specified experimental conditions.
\begin{itemize}
\item External Reward: This metric quantifies the cumulative reward that an agent receives from the environment over the course of an episode, exclusive of any adjustments or redistributions. A higher external reward is indicative of superior agent performance from RL aspect, reflecting successful task execution within the environment.

\item Blocking Times: This indicator measures the frequency with which an agent impedes the movement of other agents, including instances where an agent's presence significantly extends the optimal path length for others. A lower count of blocking times suggests better spatial awareness and consideration, contributing to smoother collective pathfinding.

\item Goal Reached: Representing the count of agents (with a maximum possible count of 8) that successfully reach their designated goals within the specified maximum steps (256 in this study), this metric directly reflects the effectiveness of the agents' pathfinding capabilities.
\end{itemize}
Table~\ref{ablation_tab} illustrates SYLPH's significant improvements across all three metrics in comparison to other variants, across diverse map types. 
It is worth noting that the reduction in blocking times highlights SYLPH's advanced planning and social behavior integration. 
Through the incorporation of social preferences and neighbor selection algorithm, agents are endowed with long-horizon coordination abilities, enabling them to proactively accommodate the movements of other agents and effectively prevent potential conflicts.
This proactive coordination behavior is a major focal point of SYLPH. 
Unlike reactive collision avoidance policy that adjust behaviors based on immediate dilemmas, SYLPH enables agents to formulate and execute a more sophisticated, forward-looking policies. 
This approach mitigates social dilemmas and conflicts, showcasing the framework's capacity for advanced, anticipatory coordination.


\subsection{Experimental Validation}

\begin{figure}[h]
	\centering
	\vspace{-0.3cm}
	\subfigtopskip=2pt 
	\subfigbottomskip=2pt 
	\subfigcapskip=-5pt
	\subfigure[random map]{
		\label{random_real}
		\includegraphics[width=0.4\linewidth]{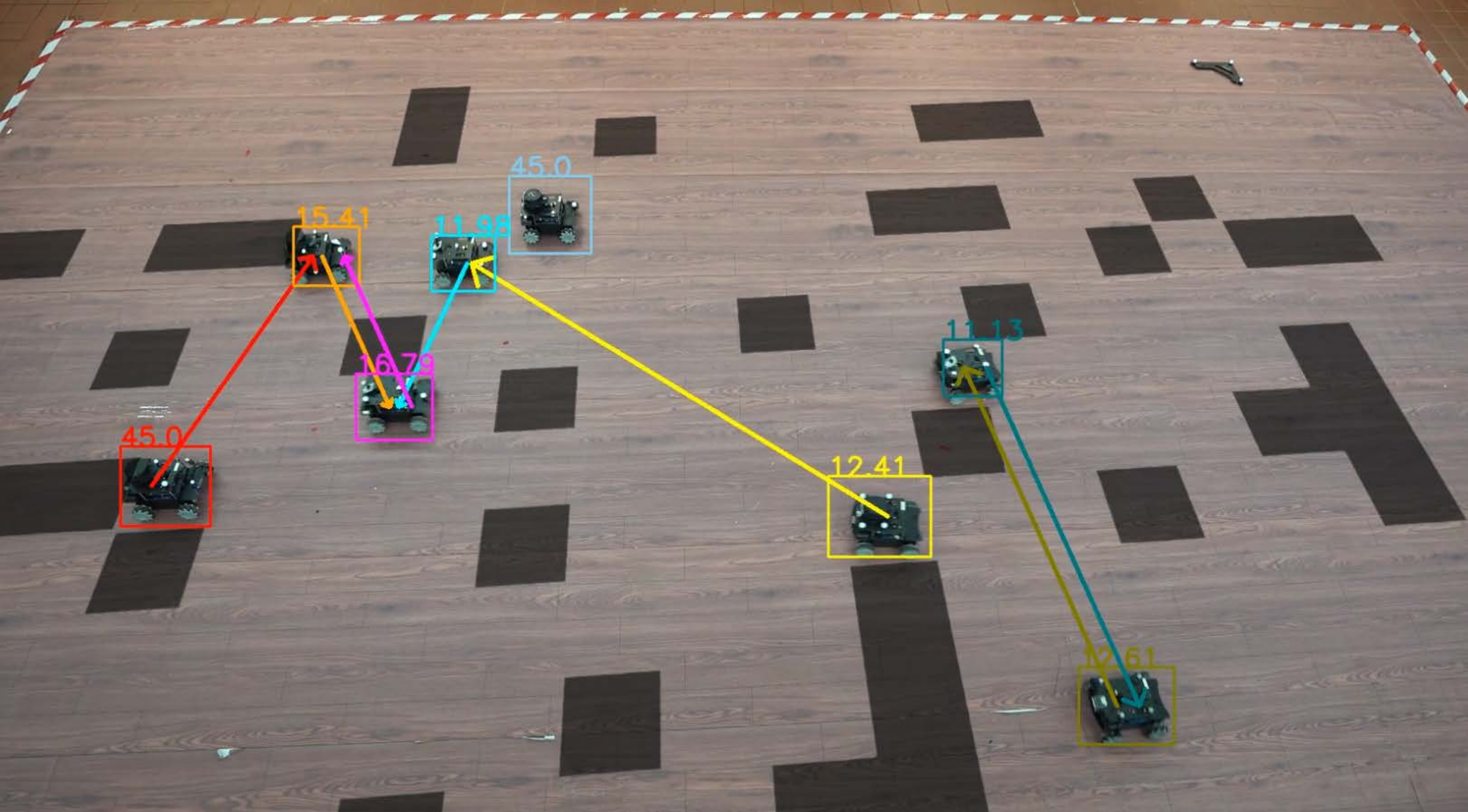}}
	\quad 
	\subfigure[room-like map]{
		\label{house_real}
		\includegraphics[width=0.4\linewidth]{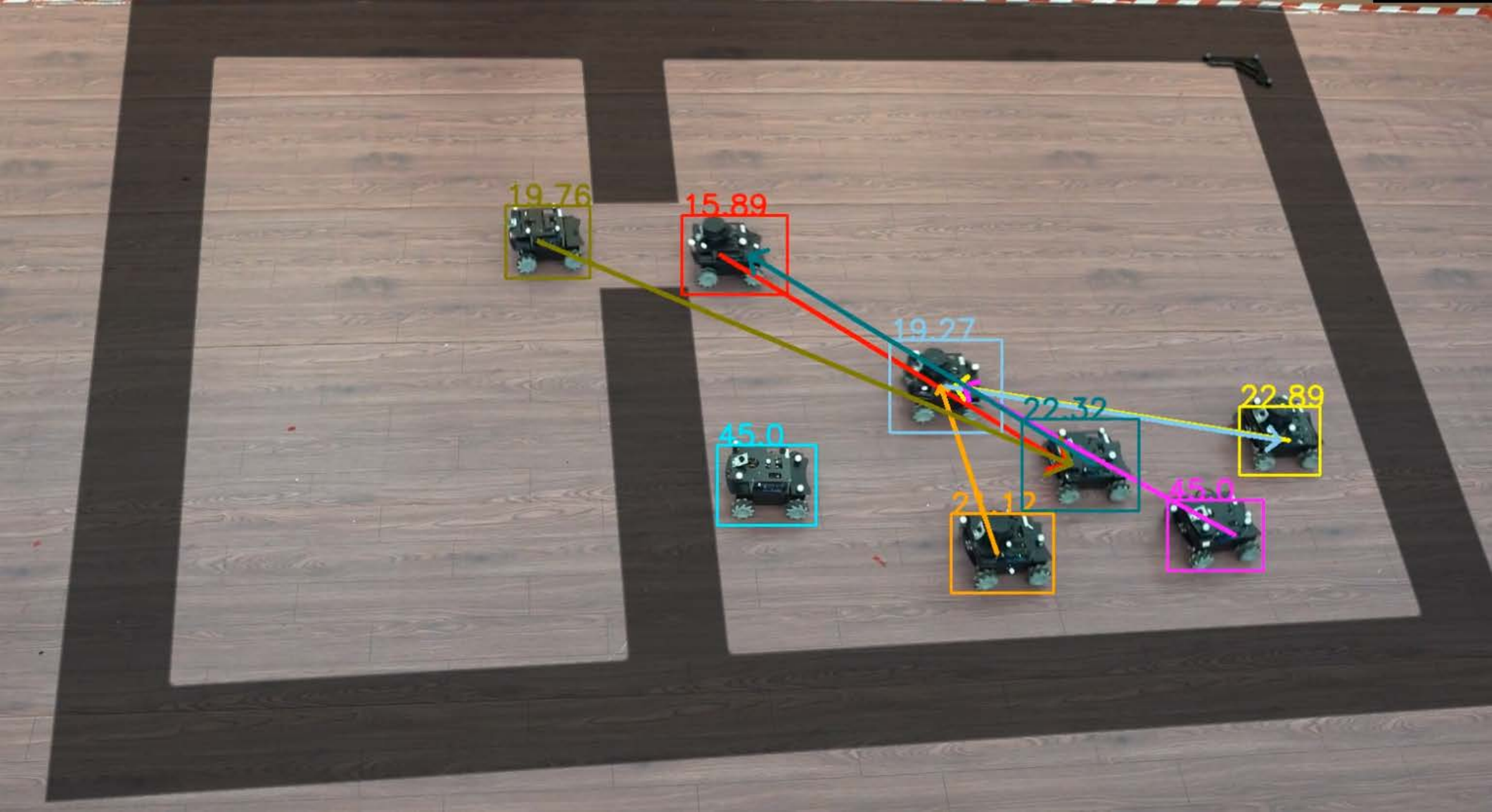}}
        \quad
	\subfigure[maze map]{
		\label{maze_real}
		\includegraphics[width=0.4\linewidth]{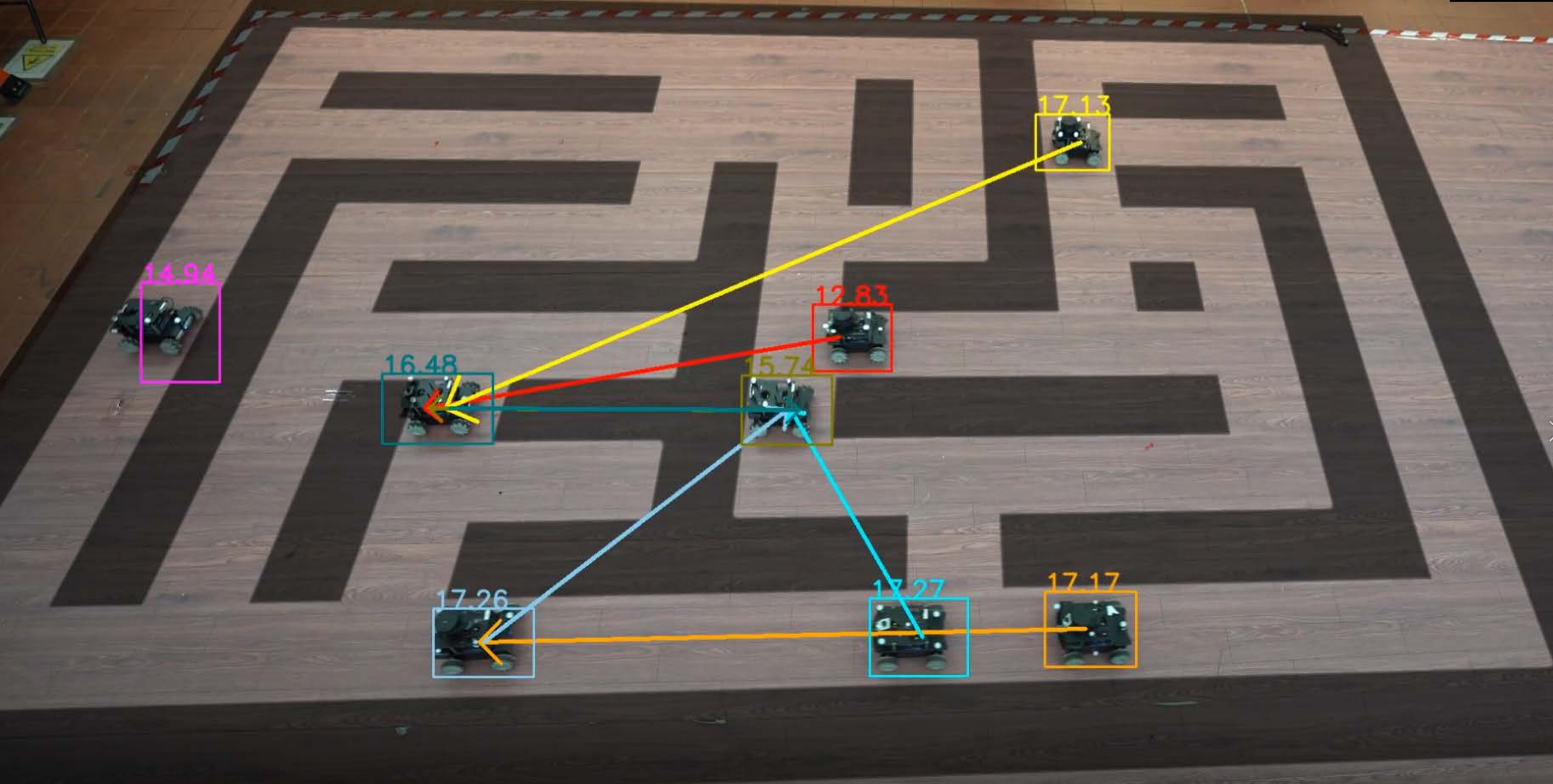}}
	\caption{Experiments with real robots on random map, room-like map, and maze map. In these figures, the black areas represent obstacles, the directed arrows indicate the partner selected by the agent, and the numbers denote the agent's current SVO (ranging from 0 to 45 degrees).}
	\label{real_exp}
\end{figure}

Fig.~\ref{real_exp} showcases our implementation of SYLPH with a team of 8 real robots pathfinding through a random map, a room-like map, and a maze map. 
For these real world pathfinding tasks, we used the standard SYLPH model trained on gridworlds without any additional finetuning/training on-robot.
To satisfy our algorithm's assumption that all agents' positions are always perfectly known, we used the \textit{Optitrack Motion Capture System} for precise localization of real robots.
We equipped the robots with Mecanum wheels to enable movement in the four cardinal directions. 
However, disturbances and control inaccuracies can cause deviations from the planned path. 
To address these issues, we employed an \textit{Action Dependency Graph} (ADG), as proposed by \cite{honig2019persistent}, which introduces a precedence order for agents occupying a cell to prevent execution errors from propagating and disrupting the overall plan.

Our experiments demonstrated that agents could reach their goals quickly and without collisions, with the ADG effectively eliminating execution errors. This highlights the potential of our method for real-world applications. Additional details about our experiment can be found in \ref{Eng_Dep}, and the full video is available in the supplementary materials.


\section{Conclusion}

This paper introduces SYLPH, a socially-aware learning-based MAPF framework designed to address potential social dilemmas by encouraging agents to learn social behaviors.
To these ends, we introduce the social value orientation (SVO) as a learnable dynamic choice for agents, to help them make decisions that benefit the group by coupling their individual interests with those of their partners.
The resulting different social preferences can break homogeneity in the team, helping couple agents directly in reward space to favor coordinated maneuvers, thereby improving cooperation among agents.
With explicit social preferences and advanced communication learning mechanism, agents are able to more effectively reason about each other’s behavior, as evidenced by the significant reduction in instances where agents blocked each other in the experiments.
Through extensive testing across various map types and agent densities, we showed that SYLPH consistently outperforms other learning-based frameworks like SCRIMP and DCC and, in some scenarios, matches the performance of traditional methods currently considered state-of-the-art.
Our framework pushes the performance boundaries of current learning-based MAPF planners and demonstrates that equipping agents with intelligent social behaviors can effectively resolve prevalent social dilemmas in MAPF problems.

Looking ahead, we plan to further refine our learning-based MAPF framework. 
Under our current training setting, SYLPH’s social preference learning enables agents to achieve $100\%$ success rate in various random environments during training. 
However, some unsolvable edge cases still arise in highly structured room-like and maze maps. 
In our future works, we will analyze these cases individually, identify their commonalities, and develop mechanisms that train agents to effectively resolve them, thus enhancing overall performance.
Additionally, the interpretability of agent behavior is another key area of interest for our future works.
We aim to establish a behavior prediction mechanism for agents by further promoting more stable SVO choices. 
We envision that this advantage may help the team create predictable and interpretable plans, which could be crucial for planning in mixed environments with humans.
That is, we envision that SVO may not only help humans understand the intentions of autonomous robots, but may also better incorporate humans in such shared environments by analyzing peoples' social preferences through the use of inverse reinforcement learning, towards improved robotic deployments in populated areas.


\section*{ACKNOWLEDGMENT}

This research was supported by the Singapore Ministry of Education (MOE), as well as by an Amazon Research Award.

\clearpage
\appendix
\section{Implementation Details}


\subsection{Hyperparameters}

Table~\ref{tab_hyperparameters} below presents the hyperparameters used to train the SYLPH model. 

\begin{table}[htp!]
\fontsize{9}{11}\selectfont
\centering
\begin{tabular}{|c | c || c | c |}
\hline
Hyperparameter & Value & Hyperparameter & Value\\  
\hline\hline
Number of agents & 8 & Value coefficient & 0.08 \\
\hline
Number of SVOs & 5 & Policy coefficient & 10 \\
\hline
Maximum episode length & 256 & Valid coefficient & 0.5\\
\hline
FOV size & 9 & Blocking coefficient & 0.5\\
\hline
FOV heuristic & 5 & Number of epochs & 10\\
\hline
World size & (10, 40) & Number of processes & 16 \\
\hline
Obstacle density & (0.0, 0.3) & Maximum number of timesteps & 2e7\\
\hline
Overlap decay & 0.95 & Minibatch size & 16\\
\hline
SVO importance factor & 2 & Imitation learning rate & 0\\
\hline
Learning rate & 1e-5 & Net size & 512\\
\hline
Discount factor & 0.95 & SVO channel size & 512\\
\hline
Gae lamda & 0.95 & Number of observation channels & 9\\
\hline
Clip parameter for probability ratio $\epsilon$ & 0.2 & Goal representation size & 12 \\
\hline
Gradient clip norm & 10 & Goal vector length & 4 \\
\hline
Entropy coefficient & 0.01 & Number of semantic tokens, \textit{L} & 16 \\
\hline
\end{tabular}
\caption{Hyperparameters table.}
\label{tab_hyperparameters}
\end{table}


\subsection{RL Reward Structure}

The reward structure for each agent's step is defined in~\ref{tab_reward} below. Similar to our previous works~\cite{wang2023scrimp,he2023alpha}, agents are penalized at each timestep unless they are on goal to promote faster episode completion. The episode terminates if all agents are on goal at the end of a timestep, or when the number of steps exceeds a pre-defined limit (256 for our model as given in~\ref{tab_hyperparameters}). 

\begin{table}[htp!]
\centering
\begin{tabular}{|c | c|}
\hline
Action & Reward \\
\hline \hline
Move (up/down/left/right) & -0.3 \\
\hline
Stay (off goal) & -0.3 \\
\hline
Stay (on goal) & 0.0 \\
\hline
Collision & -2 \\
\hline
Block & -1 \\
\hline
\end{tabular}
\caption{Reward Structure}
\label{tab_reward}
\end{table}

The blocking penalty means that if an agent occupies a space that prevents other agents from reaching their goals or significantly lengthens their optimal paths, a penalty is incurred proportional to the number of blockings caused.
The specific penalty received by an agent depends on the number of fellow agents it blocks $c$. The penalty can be represented as $-1\cdot c$, which have been used in Eq~\ref{proof_eq}.
This definition ensures that penalties are not only punitive but also proportionate to the level of inconvenience caused, encouraging agents to consider the broader implications of their decisions on system efficiency and collective goal attainment.

\section{Claim and Proof}
\label{proof_1}
\newtheorem*{remark}{Remark}
\newtheorem*{theorem}{Theorem}

Let there be an arbitrary value $a$ and $b$ selected from a finite set, and an arbitrary value $c$ such that:
\begin{equation}\label{proof_eq}
\begin{aligned}
    a &\in \{-c, -2, -0.3, 0\};\\
    b &\in \{-c, -2, -0.3, 0\};\\
    c &\in \mathbb{R}~~\textnormal{and}~~c \ge 1;
\end{aligned}
\end{equation}
where $a$ and $b$ mean the ego agent and its partner's external rewards, $R_i^{ex}$ and $R_p^{ex}$. According to Table~\ref{tab_reward}, we can figure out that $-0.3$ is the moving and staying idle (off goal) cost; $0$ is the penalty for agent standing on its goal; $-2$ denotes the collision penalty; and $-c$ is the blocking reward. 
Furthermore, we define a function $f(x)$ as:
\begin{equation}
\begin{aligned}
    f(x) = a\cdot\cos{x} + b\cdot\sin{x},~~x \in [0^\circ, 45^\circ];
\end{aligned}
\end{equation}

\begin{remark}
The function $f(x)$ models the equation of $R_i^a$ defined in Eq~\ref{reward_redistri}.
\end{remark}

\begin{theorem}
The function $f(x)$ is monotonically non-increasing $x$ for any valid values of $a$ and $b$.
\end{theorem}

\begin{proof}
For the sake of contradiction, assume that the value of the function $f(x)$ is monotonically increasing. In other words, assume that $x$ is increasing from $0^\circ$ to $45^\circ$ for any value $a$ and $b$. Then, it must be true that:
\begin{equation}
\begin{aligned}
    f'(x) = -a\cdot\sin{x} + b\cdot\cos{x} > 0
\end{aligned}
\end{equation}
Considering the domain of $x$ and the permissible values of $a$ and $b$, it must be true that:
\begin{equation}
\begin{aligned}
    -a\cdot\sin{x} \ge 0\\
    b\cdot\cos{x} \leq 0
\end{aligned}
\end{equation}
Hence, for the assumption to hold:
\begin{equation}
\begin{aligned}
    -a\cdot\sin{x} > b\cdot\cos{x} 
\end{aligned}
\end{equation}
Rearranging the inequality,
\begin{equation}
\begin{aligned}
    -a\cdot\frac{\sin{x}}{\cos{x}} &> b\\
    -\tan{x} &> \frac{b}{a}
\end{aligned}
\end{equation}
But for given domain of $x$,
\begin{equation}
\begin{aligned}
    -\tan{x} \leq 0
\end{aligned}
\end{equation}
If the above is true, it must be true that:
\begin{equation}
\begin{aligned}
    0 \ge &-\tan{x} > \frac{b}{a}\\
    \therefore 0 &> \frac{b}{a}
\end{aligned}
\end{equation}
Which is impossible from the possible values of $a$ and $b$ as they are both less than or equal to $0$. Therefore, the assumption is false and the function $f(x)$ is not monotonically increasing.

By contradiction, it must be true that the function $f(x)$ is monotonically non-increasing. In other words, the function $f(x)$ is consistent or decreasing when $x$ is increasing from $0^\circ$ to $45^\circ$ for any valid value of $a$ and $b$.

\end{proof}


\section{Random SVO Training Results}
\label{random_svo}

The experiments in this section is to verify the effectiveness and efficiency of the trained SVO policy provided by SYLPH. 
To clarify the efficiency of SYLPH, we assigned random SVOs to all agents in the comparative experiment, which also served to enhance population diversity.
Specifically, three comparative experiments were conducted to update the agents' SVO at different frequencies: 1) the SVO is set at the beginning of the episode and remains unchanged throughout, 2) the SVO remains unchanged until the agent switches partners, and 3) the SVO is randomly reset at each step.

\begin{figure}[t]
\centering
\includegraphics[width=5.2in]{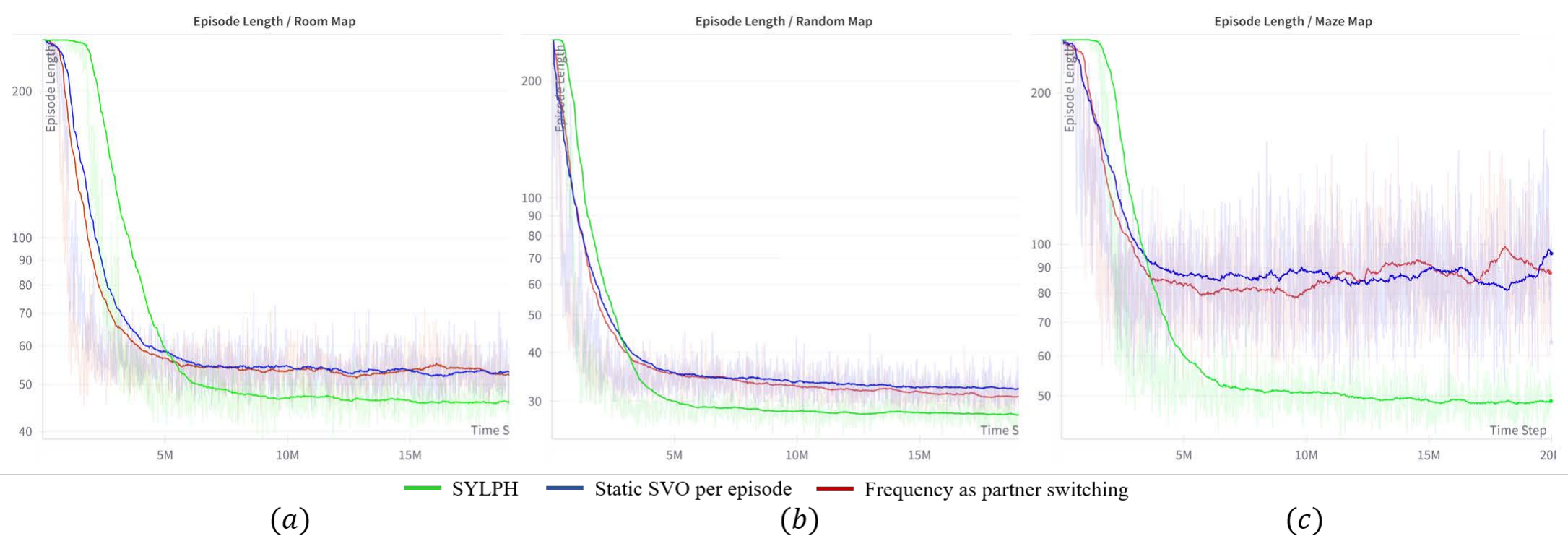}
\vspace{-0.3cm}
\caption{Training curves for SYLPH and two random SVO assignment which updates with varying frequencies across room-like maps, random maps, and maze maps.}
\label{rsvo_curves}
\vspace{-0.4cm}
\end{figure}

The outcomes of these experimental setups are depicted in Fig.~\ref{rsvo_curves}. 
Notably, the policy of resetting the SVO at each step resulted in confusion among agents about their SVO policy, preventing the training from converging to an effective action policy. 
This indicates the disruptive impact of high-frequency random SVO changes on agent behavior and decision-making.
In contrast, the other two methods of random SVO assignment allowed agents to learn effective action policies, although performance metrics slightly lower than those achieved by the SVO policy specifically learned by SYLPH in random and room-like maps. 
This slight difference in performance might be attributed to the limited number of agents (only 8) involved in the training, which constrains the extent and variety of potential social dilemmas and conflicts within these less complex environments. 
Consequently, the added value of adaptive social behaviors in these settings is somewhat restricted.
However, a different trend was observed in the training outcomes on maze maps, which are characterized by high obstacle density and highly structured obstacle distribution. 
Here, even with just eight agents, the structured environment teems with numerous social dilemmas and conflicts. 
Under these conditions, the SVO policy implemented by SYLPH demonstrated clear advantages, leading to significant performance improvements. 

The experiments validate the superiority of the SVO policy trained under the SYLPH framework over randomly assigned SVO polices.


\section{Engineering Deployment}
\label{Eng_Dep}

\subsection{Setup}

\begin{figure}[h]
	\centering
	\vspace{-0.3cm}
	\subfigtopskip=1pt 
	\subfigbottomskip=2pt 
	\subfigcapskip=-5pt
	\subfigure[random map]{
		\label{random_map_realExp}
		\includegraphics[width=0.4\linewidth]{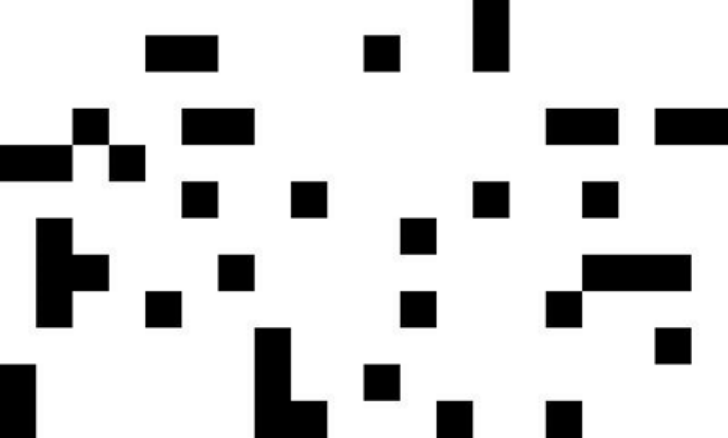}}
	\quad 
	\subfigure[room-like map]{
		\label{room_map_realExp}
		\includegraphics[width=0.4\linewidth]{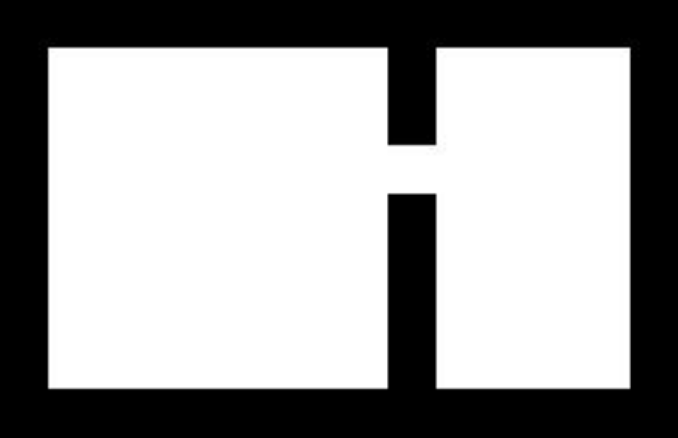}}
        \quad
	\subfigure[maze map]{
		\label{maze_map_realExp}
		\includegraphics[width=0.4\linewidth]{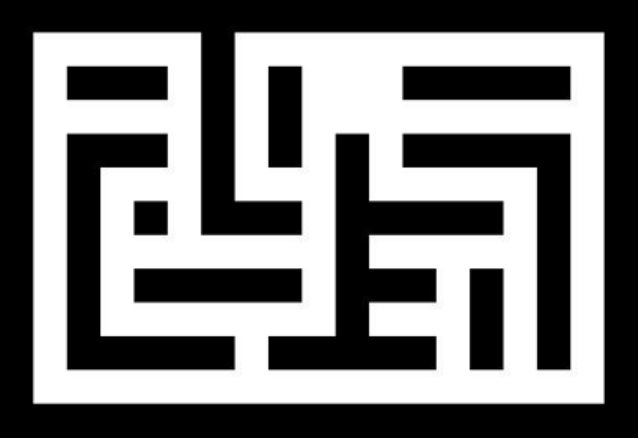}}
	\caption{Maps for real robot experiments.}
	\label{realExp_Maps}
\end{figure}
Fig.~\ref{realExp_Maps} illustrates the random map, room-like map, and maze map used in our experiment. 
When mapped to the real world, each cell has a side length of $0.3m$, which is slightly larger than the size of the agent to ensure that the agent occupies only one cell on the map.
We utilized 8 robots equipped with Mecanum wheels, each robot measuring approximately $0.23m\times 0.2m$.
We used the \textit{OptiTrack Motion Capture System} to get the accurate positions of these 8 robots. 
The configuration of the agents' starting and goal positions was randomly generated. 
In the experiment, the robots were aware of the virtual positions of obstacles and were programmed to avoid these areas. 
However, the real environment did not contain physical obstacles which can prevent interference with the line of sight of the \textit{OptiTrack} motion capture system.

\subsection{Action dependency graph}

SYLPH generates paths assuming each agent operates perfectly at every step in a discrete map. 
However, due to the imperfect nature of robots and the continuous environment, directly executing these planned paths in the real world is impractical. 
For example, a planner might instruct agent $\mathtt{A}$ to move to agent $\mathtt{B}$'s current position while agent $\mathtt{B}$ moves to a new cell. 
Executing this plan directly could result in collisions due to localization errors, delays in motor control, or differences in velocities.

To prevent such issues and ensure the feasibility of our joint set of actions, we adopt the method proposed by \cite{honig2019persistent} by constructing an \textit{Action Dependency Graph} (ADG). 
The ADG establishes a precedence order for agents occupying a cell, meaning that faster-moving agents will wait for others if the planned path requires them to occupy the cell afterward. 
This mechanism ensures that execution errors do not propagate and disrupt the overall plan.
In the above example, the ADG ensures that agent $\mathtt{A}$ will wait for $\mathtt{B}$ to vacate its current cell before moving in, thus preserving the integrity of the planned path.
Such a mechanism ensures that the planned path is executed securely between agents, though it may introduce slight delays.

\subsection{Execution}

To implement the ADG, each robot's action must be converted into a task. We define a \textit{Task} object with the following attributes:
\begin{verbatim}
object Task {
    taskID;       // Unique identifier for the task
    robotID;      // ID of the robot assigned to the task
    action;       // Action to be performed
    startPos;     // Initial position of the robot
    endPos;       // Position after action completion
    time;         // Scheduled time for the action
    dependencies; // All Tasks that need to be completed before this
    status;       // Current status: staged, enqueued, or completed
};
\end{verbatim}
During execution, we iterate through all agents, translating their actions into \textit{Task} objects and constructing the ADG from these tasks. Each task can have one of three statuses: \textit{STAGED}, \textit{ENQUEUED}, or \textit{DONE}.
Initially, tasks are set to \textit{STAGED}. When all \textit{dependencies} of the task are completed, the status changes to \textit{ENQUEUED}, meaning readiness for execution. 
The task status updates to \textit{DONE} once the agent reaches the specified \textit{endPos}.

At the ROS execution level, a \textit{central node} is responsible for generating and maintaining the ADG, as well as distributing tasks to robots when they are ready to be enqueued. 
Each robot is equipped with a \textit{robot node}, which handles receiving tasks from the central node, extracting goals from these tasks, and using a PID controller to reach those goals.
Once a robot reaches its goal, it communicates this achievement to the central node, marking the task as \textit{DONE}. 
The central node then uses the ADG to enqueue additional tasks for the robots.




\clearpage
\bibliographystyle{elsarticle-num-names}
\bibliography{references}




\end{document}

%% file: time_consuming.tex
\begin{tabular}{c|cc|cc|cc}
    \toprule              
     Cases  & \multicolumn{2}{c}{$radnom$}               & \multicolumn{2}{c}{\textit{room-like}}   & \multicolumn{2}{c}{\textit{maze}}         \\
     Configuration & \multicolumn{2}{c}{\textit{32-32-0.2-200}} & \multicolumn{2}{c}{\textit{32-32-0.3-100}} & \multicolumn{2}{c}{\textit{32-32-0.5-32}} \\
    \midrule\midrule                                 
    Time Type & General & Success & General & Success & General & Success \\
    \midrule\midrule
    SYLPH    &  100.944$s$  &  79.621$s$  &  48.643$s$   & 35.533$s$  &  5.867$s$   &  4.229$s$   \\
    \midrule
    SCRIMP   &  627.914$s$  & 402.702$s$  &  422.171$s$  & 127.057$s$ &  14.635$s$  &  9.967$s$   \\
    \midrule
    DCC      & 218.786$s$  & --   & 56.918$s$   &  36.728$s$  &  5.427$s$   & 1.891$s$   \\
    \midrule
    EECBS    & 2.753$s^\blacktriangle$  & 60.001$s^\blacktriangle$   & 1.211$s^\blacktriangle$ & 52.515$s^\blacktriangle$  &  1.544$s^\blacktriangle$   & 4.247$s^\blacktriangle$   \\
    \midrule
    LNS2     & 1.620$s^\blacktriangle$  & 0.429$s^\blacktriangle$  &  2.635$s^\blacktriangle$   & 2.496$s^\blacktriangle$  &  0.029$s^\blacktriangle$  & 0.029$s^\blacktriangle$  \\
    \midrule
    LACAM    & 30.023$s$  & 28.120$s$   & 30.012$s$   & 20.015$s$  &  30.010$s$   & 29.319$s$   \\
    \bottomrule
\end{tabular}

%% file: paired_t_test.tex
\begin{tabular}{c|ccc}
    \toprule              
     Paired  & \textit{random}  & \textit{room-like} & \textit{maze}        \\
     t-test & \textit{32-32-0.2-200} & \textit{32-32-0.3-100} & \textit{32-32-0.5-64} \\
    \midrule\midrule
    Static SVO Per Episode            &  $\textit{p}<0.001$  &  $\textit{p}<0.001$  &  $\textit{p}<0.001$   \\
    \midrule
    Static SVO Per Partner   &  $\textit{p}<0.001$  &  $\textit{p}<0.001$  &  $\textit{p}<0.001$  \\
    \bottomrule
\end{tabular}

%% file: tab_ablation.tex
\begin{tabular}{c|ccc|ccc}
    \toprule              
    \textbf{Map} & \textbf{Semantic}   & \textbf{Social}   & \textbf{Tie-} & \textbf{External}  & \textbf{Blocked}    & \textbf{Goals}         \\
    \textbf{Type} & \textbf{Transformer}& \textbf{Behavior} & \textbf{Breaking} & \textbf{Reward}$\uparrow$ & \textbf{Agents}$\downarrow$ & \textbf{Reached} $\uparrow$ \\
    \midrule\midrule                                 
    &    --     & --    & --     &  -57.490    &    6.904      & 7.924      \\
    \cmidrule{2-7}
    Random & \checkmark  & --   & --   & -51.462  &  2.751   & 7.962   \\
    Map & \checkmark     & \checkmark   & --  &  -46.926     &    0.5562    &  7.986      \\
    & \checkmark    & \checkmark   & \checkmark   &  \color{red}\textbf{-39.110}   &   \color{green}\textbf{0.3345}    & \color{blue}\textbf{7.999}  \\
    \midrule\midrule                                 
    &    --     & --    & --    &  -85.358     &   3.133    & 7.87          \\
    \cmidrule{2-7}
    Room-like & \checkmark          & --      & --      &  -81.154      &    2.100          & 7.905              \\
    Map & \checkmark          & \checkmark          & --     &  -74.443         &    0.4126             & 7.932         \\
    & \checkmark          & \checkmark          & \checkmark          &   \color{red}\textbf{-61.210}        &  \color{green}\textbf{0.2183}       & \color{blue}\textbf{7.993}       \\
    \midrule\midrule                                 
    &    --     & --     & --      &  -158.168         &   27.168    & 7.382             \\
    \cmidrule{2-7}
    Maze & \checkmark          & --                  & --        &  -126.976   &    14.451    & 7.557                \\
    Map & \checkmark          & \checkmark          & --         & -90.484       &   1.709    & 7.836                \\
    & \checkmark          & \checkmark          & \checkmark          &  \color{red}\textbf{-69.69}         &  \color{green}\textbf{0.5189}    & \color{blue}\textbf{7.977}   \\
    \bottomrule
\end{tabular}